\documentclass[journal]{IEEEtran}

\usepackage[ruled]{algorithm2e}
\usepackage{graphicx}
\usepackage{amsmath}
\usepackage{multirow} 
\usepackage{cite}
\usepackage{amsmath,amssymb,amsfonts}
\usepackage{algorithmic}
\SetKwComment{Comment}{/* }{*/}
\usepackage{textcomp}
\usepackage{makecell}
\usepackage{pifont}
\usepackage{xcolor}
\usepackage{amsthm}
\usepackage[most]{tcolorbox}
\usepackage{xcolor}

\definecolor{lightgray}{gray}{0.96}

\usepackage{booktabs}
\usepackage[colorlinks,
            linkcolor=blue,       
            anchorcolor=blue,  
            citecolor=blue,      
            ]{hyperref}
\usepackage{amsmath,amssymb,amsthm}
\newtheorem{definition}{Definition} 
\newtheorem{theorem}{\bf Theorem}

\newtheorem{proposition}{\bf Proposition}

\usepackage[ruled]{algorithm2e}
\def\BibTeX{{\rm B\kern-.05em{\sc i\kern-.025em b}\kern-.08em
    T\kern-.1667em\lower.7ex\hbox{E}\kern-.125emX}}

\tcolorboxenvironment{definition}{
  colback=lightgray,
  boxrule=0pt,
  frame hidden,
  arc=0pt,
  left=0pt,
  right=0pt,
  top=0pt,
  bottom=0pt,
}

\tcolorboxenvironment{proposition}{
  colback=lightgray,
  boxrule=0pt,
  frame hidden,
  arc=0pt,
  left=0pt,
  right=0pt,
  top=0pt,
  bottom=0pt,
}

\tcolorboxenvironment{theorem}{
  colback=lightgray,
  boxrule=0pt,
  frame hidden,
  arc=0pt,
  left=0pt,
  right=0pt,
  top=0pt,
  bottom=0pt,
}

\begin{document}

\title{Towards Personalized Differentially Private Learning for Decentralized Local Graphs}

\author{Longzhu~He, Peng~Tang, Chaozhuo Li, Jinhu~Fu, Litian Zhang, Li~Sun, Philip S. Yu, and Sen~Su 
\thanks{This work was supported by the National Key Research and Development Program of China (No. 2024YFF0907401), the National Natural Science Foundation of China (No. 62072052), and the BUPT Excellent Ph.D. Students
Foundation (No. CX20260030), and the Shandong Provincial Natural Science Foundation (No. ZR2025MS1038). (Corresponding author: Sen Su.)}
\thanks{Longzhu He, Jinhu Fu, Li Sun, and Sen Su are with
the School of Computer Science, Beijing
University of Posts and Telecommunications, Beijing 100876, China, and also with the State Key Laboratory of Networking
and Switching Technology, Beijing University of Posts and Telecommunications, Beijing 100876, China. (e-mail: helongzhu@bupt.edu.cn; fjhu@bupt.edu.cn; lsun@bupt.edu.cn, susen@bupt.edu.cn).}
\thanks{Peng Tang is with the School of Cyber Science and Technology, Shandong University, Qingdao, China (e-mail: tangpeng@sdu.edu.cn).}
\thanks{Chaozhuo Li and Litian Zhang are with the School of Cyberspace Security, Beijing University of Posts and Telecommunications, Beijing 100876, China (e-mail: lichaozhuo@bupt.edu.cn; litianzhang@bupt.edu.cn).}
\thanks{Philip S. Yu is with the Department of Computer Science, University of Illinois at Chicago, Chicago, IL 60607 USA (e-mail: psyu@uic.edu).}
}

\markboth{IEEE Transactions on Knowledge and Data Engineering}%
{Shell \MakeLowercase{\textit{et al.}}: Bare Demo of IEEEtran.cls for IEEE Journals}

\maketitle

\begin{abstract}
Graph-structured data is increasingly generated and stored in decentralized environments, such as social platforms, mobile applications, and edge networks, where users maintain control over their local graph data. However, collecting and analyzing such decentralized graph data for downstream learning tasks raises significant privacy concerns, as nodes and their attributes often contain sensitive personal information. Local Differential Privacy (LDP) has emerged as a promising solution for privacy-preserving data collection without relying on trusted servers. Nevertheless, existing LDP-based graph learning methods typically assume uniform privacy requirements across users, ignoring the heterogeneous and personalized privacy preferences commonly observed in real-world systems. This uniform treatment leads to inflexible noise injection at the data collection stage, resulting in substantial distortion of graph data and degraded utility in subsequent analysis. To address this limitation, we propose PPGNN, a personalized differentially private framework for decentralized graph data. PPGNN enables user-specific privacy budgets during local perturbation while preserving analytical utility. To handle heterogeneous privacy levels and noise distortion, we design a two-stage solution consisting of a Personalized Perturbation Mechanism (PPM) and a weighted calibration strategy, FlexProp. Extensive experiments on six real-world graph datasets demonstrate that PPGNN effectively balances personalized privacy protection and data utility in decentralized graph learning scenarios.
\end{abstract}

\begin{IEEEkeywords}
local differential privacy, graph-structured data, graph learning, personalized privacy requirements
\end{IEEEkeywords}

\IEEEpeerreviewmaketitle

\section{Introduction}
\begin{figure}
  \centering
  \includegraphics[scale=0.6]{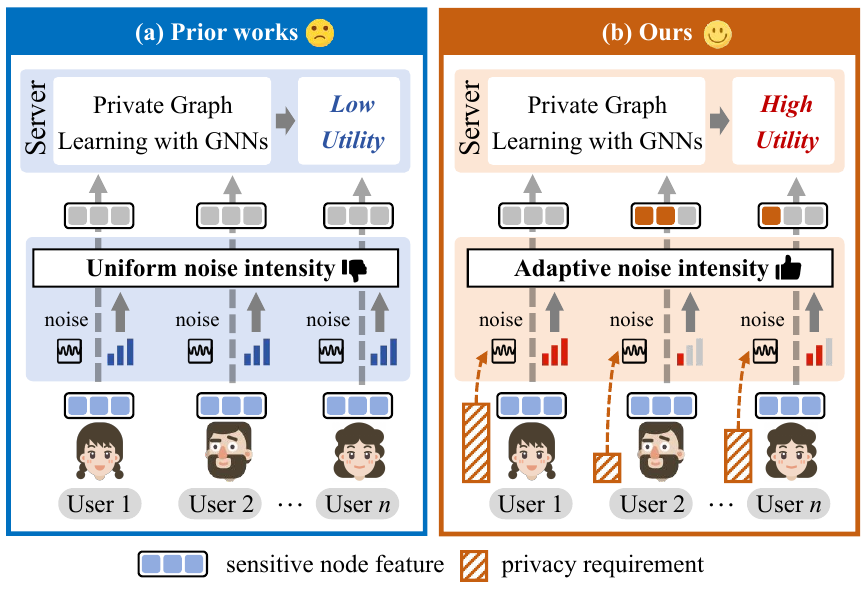}
  \caption{Comparison of \textbf{(a) prior works} and \textbf{(b) ours} in the locally private graph learning scenario. The scenario involves a server (\textit{e.g.}, a social network server) and multiple decentralized users, each with sensitive node features and distinct privacy requirements. The server collects perturbed data and performs private graph learning using a GNN. \textbf{(a)} Previous works overlook personalized privacy requirements of users and \textit{apply uniform noise intensity} to protect node features, leading to \textit{low utility} in private graph learning. \textbf{(b)} In contrast, our approach \textit{adapts the noise intensity} according to user-specific privacy requirements, thereby promoting \textit{high utility} in private graph learning.}
  \label{fig:0}
\end{figure}

\IEEEPARstart{G}{raph}-structured data has become a fundamental data type in modern information systems, widely used to model complex relationships in social networks, mobile applications, and online platforms. In many real-world scenarios, such graph data is decentralized and maintained locally by users, containing sensitive information embedded in node attributes and structural connections. Collecting and aggregating these local graph data for downstream analytical tasks may therefore raise significant privacy concerns~\cite{wang2022group}. Graph Neural Networks (GNNs)~\cite{hamilton2017inductive,wu2020comprehensive} have emerged as a powerful tool for learning from graph data, demonstrating remarkable performance in various downstream tasks such as node classification~\cite{DBLP:conf/iclr/KipfW17}, link prediction~\cite{zhang2018link}, and graph classification~\cite{DBLP:conf/iclr/XuHLJ19}. However, applying GNNs in decentralized environments requires collecting user data to a central server, which further amplifies privacy risks. This is particularly evident in social mobile applications like Instagram and WhatsApp~\cite{paul2014survey}, where servers aim to gather various types of information (\textit{\textit{e.g.}}, user profiles and interactions) for recommendation purposes. Recent studies have also shown that GNNs are vulnerable to various privacy attacks~\cite{shen2022model,wang2022group,wu2022model,zhang2022inference,zhuangunveiling,zhang2022model}, leading to potential leakage of sensitive information. Therefore, there is an urgent need for innovative privacy-preserving graph learning models.

To mitigate privacy threats, local differential privacy (LDP)\cite{dwork2006calibrating} has emerged as a pivotal and well-recognized method for safeguarding sensitive data in distributed systems. LDP~\cite{dwork2006calibrating} offers robust privacy guarantees for secure data collection and analysis, and has been widely adopted by major companies such as Google~\cite{erlingsson2014rappor}, Apple~\cite{patent1}, and Microsoft~\cite{ding2017collecting}. In the LDP setting, there is a server and multiple user clients, where it is assumed that the server is untrustworthy. Each user ensures privacy by locally perturbing their data (typically through noise injection~\cite{zhou2021local}, the magnitude of which depends on the \textit{privacy budget}) before sending it to the server. Given its strong data privacy guarantees, LDP has become a key approach to protecting decentralized local graph node features, which are often sensitive. Recent research has delved into graph learning under LDP for enhanced privacy preservation, including advanced LDP mechanisms, such as the multi-bit mechanism~\cite{sajadmanesh2021locally,lin2022towards,du2021calibrating}, as well as classical LDP mechanisms like the Laplace mechanism~\cite{dwork2006calibrating}, Gaussian mechanism~\cite{balle2018improving}, and 1-bit mechanism~\cite{ding2017collecting}. 
Fig.~\ref{fig:0} illustrates the local private graph learning scenario, where each user has a sensitive node feature. The server collects the perturbed data and performs private graph learning on the noisy data.

However, as shown in Fig.~\ref{fig:0}(a), prior studies~\cite{sajadmanesh2021locally,lin2022towards,pei2023privacy,du2021calibrating,he2025going,he2026devil,he2026devil1,he2026differentially} largely overlook the fact that users may have different overall levels of privacy  requirement. In real-world scenarios, privacy expectations can vary significantly across users due to factors such as gender, geography, and age. For instance, research~\cite{coopamootoo2022feel} shows that females tend to express higher privacy concerns than males, revealing a “\textit{privacy gender gap}” in the context of online tracking~\cite{englehardt2016online}. Similarly, individuals in the UK are reported to be less likely to disclose negative emotions than those in Germany or France, suggesting a “\textit{privacy geography gap}.” These findings indicate that certain user groups may prefer stronger privacy protection across all of their personal data. Despite this, existing locally differentially private graph learning methods assume a uniform privacy level for all users. As shown in Fig.~\ref{fig:0}(a), this leads to the application of a global maximum noise intensity determined by the most privacy-sensitive user to all node features. Such a one-size-fits-all strategy significantly degrades model utility, as excessive noise added to satisfy a small subset of users unnecessarily undermines the effectiveness of the overall learning process.

To address the above issue, as illustrated in Fig.~\ref{fig:0}(b), we propose a novel privacy-preserving graph learning framework grounded in the notion of \textit{personalized} LDP (PLDP)~\cite{chen2016private}. This framework adaptively adjusts noise intensity to meet users' privacy requirements while enhancing the utility of private graph learning. PLDP, an extension of the traditional LDP, is designed to offer varying levels of privacy protection tailored to individual users or data points. Introducing a personalized definition to LDP enables dynamic adjustment of privacy parameters based on factors such as importance, sensitivity, or specific user requirements.

However, under this personalized privacy setting, two key challenges arise. On the one hand, dual protection of privacy elements becomes essential: not only must node features be safeguarded, but the privacy level~\cite{yiwen2018utility} of each user also requires protection. A user's privacy level reflects their valuation of privacy, which can inadvertently reveal sensitive personal attributes that necessitate additional safeguarding. Protecting a user's privacy level necessitates designing a specialized LDP protection algorithm and addressing the interplay between node feature protection and privacy levels. Improper handling of this interaction risks disclosing the user's privacy preferences, making the synergistic protection of both aspects a critical consideration. In this dual privacy setting, rationally allocating the privacy budget between the two elements to maximize the utility of privacy-preserving graph learning is vital. On the other hand, noise calibration in multilevel privacy scenarios poses additional challenges. Directly training GNNs with perturbed features can severely reduce the utility of privacy graph learning due to excessive noise. While previous works attempt denoising via multi-hop aggregation~\cite{morris2019weisfeiler,abu2019mixhop,huang2024higher}, this approach proves suboptimal in a multilevel privacy setting because perturbation levels differ for each node. Simple multi-layer aggregation fails to exploit the effective information from each node fully. Consequently, achieving accurate noise calibration under a personalized privacy framework introduces greater complexity than non-personalized approaches.

In this paper, we propose a novel locally differentially private graph learning
framework, PPGNN (\underline{P}ersonalized \underline{P}rivacy-preserving \underline{G}raph \underline{N}eural \underline{N}etwork), which integrates two core components: the Personalized Perturbation Mechanism (PPM) and the FlexProp algorithm. Specifically, 
PPM is designed to achieve dual privacy protection for node features and privacy levels. It comprises two fundamental elements: the Multi-dimensional Local Randomizer (MLR) and the Extended Square Wave (ESW) mechanism. The MLR protects multi-dimensional node features and ensures their utility through dimensionality reduction and rigorous theoretical considerations. Meanwhile, the ESW mechanism enhances privacy protection by extending the capabilities of the traditional square wave mechanism~\cite{li2020estimating} to cover users' privacy levels across discrete domains, ensuring that the protection is effective and adaptable to varying user requirements. To tackle the challenges of calibrating the perturbed node features, we introduce a weighted aggregation algorithm called FlexProp. FlexProp is essential for integrating individual users' privacy levels and aggregating information from neighboring nodes within $K$ hops. By applying varying weights based on privacy levels, FlexProp facilitates nuanced denoising, thus improving data quality while maintaining robust privacy measures. This aggregation process supports privacy-preserving graph learning, ensuring the GNN performs effectively despite the added noise. Notably, our PPGNN framework is designed to be compatible with various GNN architectures, providing both flexibility and adaptability for different application scenarios. Our contributions are summarized as follows:
\begin{itemize}
    \item \textit{Problem}. We propose a novel and realistic problem in differentially private GNNs, aiming to achieve personalized privacy preservation for users' local data.
    \item \textit{Methodology}. We introduce PPGNN, a novel locally differentially private GNN model designed to meet users' personalized privacy requirements while ensuring the utility of privacy-preserving graph learning.
    \item \textit{Experiments}. We conduct extensive experiments on six real-world datasets and three GNN architectures (GCN, GraphSAGE, GAT), demonstrating that PPGNN consistently outperforms strong baselines and achieves utility close to the non-private upper bound in multiple settings.
\end{itemize}

\textbf{Organization.} 
The rest of this paper is organized as follows. Section~\ref{S2} introduces background knowledge and formulates the problem. Section \ref{S4} describes our personalized privacy-preserving graph learning framework, PPGNN. Section~\ref{S5} presents extensive experimental results. Section~\ref{S6} reviews the related literature, and Section~\ref{S7} concludes the paper.

\begin{table}[t]
	\centering
	\caption{Notations}
	\label{table:1}
	\setlength{\tabcolsep}{5mm}\begin{tabular}{|c|c|}
		\hline
		\textbf{Symbol}  & \textbf{Description}   \\
		\hline
    $|\mathcal{V}|$ & the number of users \\\hline
     $\mathbf{A}$ & the adjacency matrix  \\\hline
      $\mathbf{X}$ & the node feature matrix        \\\hline
       $\tau$ & the privay level of users         \\\hline
    $d$ & the number of feature dimensions          \\\hline
     $m$ & the number of sampled dimensions          \\\hline
     $\epsilon$ & the privacy budget          \\\hline
   $\mathcal{M}$ & the perturbation mechanism   \\\hline
    $\mathbf{x}_v$ & the original feature vector of user $v$  \\\hline
$\mathbf{x}_v^\prime$ & the perturbed feature vector of user $v$  \\\hline
      $\mathcal{N}(v)$ & the set of neighbors of $v$ (including $v$ itself)         \\\hline
       $\mathbf{h}_v$ & the original embedding of node $v$          \\\hline
     $\widehat{\mathbf{h}}_v$ & the estimated embedding of node $v$          \\\hline
     $\gamma$ & \makecell[c]{the privacy budget allocation ratio between PPM \\and FlexProp modules} \\\hline
    $w_{ij}$ & \makecell[c]{the aggregation weight from node $j$ to node $i$ \\in weighted message passing} \\\hline
    \textbf{$g$} &  the user’s true privacy level index (input to ESW) \\\hline
 \textbf{$g'$} &  the randomized output of ESW \\\hline
 \textbf{$f$} &  the size of the input domain in ESW \\\hline
 \textbf{$b$} &  the smoothing window radius in ESW \\\hline
	\end{tabular}
\end{table}

\section{Preliminaries}\label{S2}
In this section, we first provide a formal definition of the problem (§\ref{sec21}). Next, we present essential background knowledge related to graph neural networks (§\ref{sec22}) and local differential privacy (§\ref{sec23}), followed by an introduction to the four primary node feature LDP mechanisms (§\ref{3.1}). We summarize the important notations of our paper in Table~\ref{table:1}.

\vspace{-1em}
\subsection{Problem Definition}\label{sec21} 
We provide a formal definition for the problem of learning a GNN with node data privacy. Consider a graph $\mathcal{G}=(\mathcal{V}, \mathbf{A}, \mathbf{X})$, where $\mathcal{V}=\{v_1,v_2,\dots,v_{|\mathcal{V}|}\}$ is the node/user\footnote{This paper uses ‘node’ and ‘user’ interchangeably, as in many applications, such as social networks, each node corresponds to a user.} set, and $|\mathcal{V}|$ represents the total number of users within the network. Each user $v_i$ holds locally a privacy feature vector $\mathbf{x}_i\in\mathbb{R}^d$, where $d$ denotes the dimension of the node feature vector. In graph $\mathcal{G}$, the adjacency matrix $\mathbf{A}\in\mathbb{R}^{|\mathcal{V}|\times |\mathcal{V}|}$ encodes the connections between nodes, while the node feature matrix $\mathbf{X}$ is defined as $\{\mathbf{x}_1,\mathbf{x}_2,\dots,\mathbf{x}_{|\mathcal{V}|}\}$. Consistent with prior works~\cite{sajadmanesh2021locally,lin2022towards,pei2023privacy,du2021calibrating}, we focus on the node classification task, which serves as a fundamental building block for various applications in social networks and beyond. Specifically, given a labeled node set $\mathcal{V}_l\subset \mathcal{V}$, containing nodes with known class labels drawn from the set $\mathcal{Y}=\{y_1,y_2,\dots,y_c\}$, along with an unlabeled node set $\mathcal{V}_u\subset \mathcal{V}/\mathcal{V}_l$, the objective of node classification is to accurately assign each node $v$ to one of the predefined classes within $\mathcal{Y}$. This task serves as the foundation for numerous real-world applications, including social influence analysis, content recommendation, and predictive behavior modeling.

We assume that the server is an untrusted party, which has access to the node set $\mathcal{V}$ and the adjacency matrix $\mathbf{A}$, but cannot directly observe the sensitive node feature matrix $\mathbf{X}$. The feature data is decentralized and privately held by individual users. Protecting this information is essential for maintaining user confidentiality and mitigating the risk of data leakage.\footnote{While this paper focuses on the privacy of node features, our framework is modular and can be seamlessly extended to incorporate edge privacy mechanisms such as those in~\cite{zhu2023blink,hidano2022degree}.} Each user is associated with a personalized privacy level $\tau \in \{1, 2, \dots, h\}$, where a smaller value of $\tau$ indicates a stronger privacy preference, and thus requires injecting more noise. Users independently perturb their feature vectors based on their assigned $\tau$ using a LDP mechanism, enabling personalized privacy preservation. Importantly, although we focus on feature-level privacy, the graph structure, captured by the adjacency matrix $\mathbf{A}$, remains intact and is fully utilized during training. The server collects the perturbed features and performs graph learning using GNNs, which leverage the underlying topology for message passing and representation learning. Therefore, the graph component remains integral to the framework, and the interplay between noisy features and structural information is central to the model’s design.

\vspace{-1em}
\subsection{Graph Nerual Networks}\label{sec22} 
GNNs~\cite{hamilton2017inductive,wu2020comprehensive,he2026towards,he2026conflict} learn new node representations by combining initial node features and graph topology by aggregating node and neighbor information. These representations are then used for downstream machine learning tasks like node classification. A typical $K\text{-layer}$ GNN consists of $K$ graph convolution layers, where each layer aggregates information from a node's neighbors and updates the node's representation. After $K$ aggregation iterations, the representation of a node captures the structural information in its $K\text{-hop}$ neighborhood. The $k\text{-th}$ layer of the GNN can be defined formally as follows:
\begin{align}
	\mathbf{h}^{k}_{\mathcal{N}(v)} & = \textsc{Agg}_k\left(\{\mathbf{h}_u^{k-1}, \forall u \in \mathcal{N}(v)\}\right), \label{eq:agg} \\
	\mathbf{h}^{k}_{v}     & = \textsc{Upd}_k\left(\mathbf{h}^{k}_{\mathcal{N}(v)}, \Theta^k \right), \label{eq:na}
\end{align}
where $\mathcal{N}(v)$ is the set of neighbors of node $v$ as well as could include $v$ itself, and $\mathbf{h}_u^{k-1}$ is the embedding of an adjacent node $u$ in the $k-1$ layer. $\textsc{Agg}_k(\cdot)$ denotes the aggregator function, such as  $\textsc{Sum}$, $\textsc{Mean}$, or $\textsc{Max}$. $\textsc{Upd}_k(\cdot)$ denotes a learnable update function, such as a neural network, parameterized by $\Theta^k$. The above structured approach allows GNNs to effectively leverage both node features and relational information inherently present in the graph, thereby enhancing their performance on various tasks.

\vspace{-1em}
\subsection{Local Differential Privacy}\label{sec23}  
LDP, as a variant of DP, has been adopted as a powerful privacy-preserving method by companies such as Google~\cite{erlingsson2014rappor}, Apple~\cite{patent1} and Microsoft~\cite{ding2017collecting}, and hundreds of millions of users' private information is being protected by this technology. In the LDP setting, the user perturbs the original data using a randomizer $\mathcal{R}$ on the user side and uploads the perturbed data to the server side. This ensures the privacy of user data, as only the data owner can access the original data~\cite{yang2020local}. The randomizer $\mathcal{R}$ is defined as follows:
\begin{definition}[$\epsilon\text{-LDP}$~\cite{dwork2014algorithmic,yang2024local}]\label{def:1}
A randomizer $\mathcal{R}$ satisfies $\epsilon\text{-LDP}$, where $\epsilon > 0$, if and only if for any user's private data $x$ and $x^\prime$, and for all possible outputs $y \in Range(\mathcal{R})$:
	\begin{equation}
		\Pr[\mathcal{R}(x) = y] \le e^\epsilon \cdot \Pr[\mathcal{R}(x^\prime) = y].
	\end{equation} 
\end{definition}
The parameter $\epsilon$, known as the \emph{privacy budget}, plays a crucial role in LDP by regulating the balance between utility and privacy. A smaller (resp. larger) $\epsilon$ value provides stronger (resp. weaker) privacy guarantees but may result in lower (resp. higher) utility. The above definition implies that no matter what side knowledge the adversary has, they cannot, with high probability, infer the input value by observing the output. In addition, LDP satisfies several important properties~\cite{dwork2014algorithmic} that are essential for subsequent proofs, as follows:
\begin{proposition}[Sequential Composition\cite{dwork2014algorithmic}]\label{pos1}
Consider a sequence of computations $\hspace{-0.2em}\mathcal{A}_1,\dots,\mathcal{A}_k$ applied to different components of the same local data record. If each $\mathcal{A}_i$ satisfies $\epsilon_i$-LDP, then performing these computations sequentially ensures an overall privacy guarantee of $(\sum_i \epsilon_i)$-LDP for that record.
\end{proposition}
\begin{proposition}[Post-processing\cite{dwork2014algorithmic}]\label{pos2}
If an algorithm $\hspace{-0.2em}\mathcal{A}(\cdot)$ satisfies $\hspace{-0.2em}\epsilon\text{-LDP}$, then any further processing of its output by another algorithm $\mathcal{B}(\cdot)$ (i.e., $\mathcal{B}(\mathcal{A}(\cdot))$) also maintains $\epsilon\text{-LDP}$.
\end{proposition}
In real-world scenarios, users often exhibit varying levels of privacy sensitivity and expectations, resulting in diverse privacy requirements. The standard LDP framework assumes a uniform privacy budget for all users, which limits its applicability in personalized settings. To address this limitation, the notion of personalized LDP (PLDP) has been proposed~\cite{chen2016private,yiwen2018utility,murakami2019utility,duan2024ldptube}. Under the PLDP setting, the server provides a set of discrete privacy levels, and each user can independently select a level $\tau \in \{1, \dots, h\}$ based on their individual privacy preferences. This selected level corresponds to a user-specific privacy budget $\epsilon^\tau$, allowing for flexible control over the noise injection process while respecting user autonomy. Importantly, despite this personalization, PLDP retains the formal privacy guarantees and desirable properties of standard LDP~\cite{dwork2014algorithmic}. 

\vspace{-1em}
\subsection{Four Primary Node Feature LDP Mechanisms}\label{3.1}
\subsubsection{Laplace mechanism}
The Laplace mechanism~\cite{dwork2006calibrating} is a classical LDP mechanism known for its simplicity. Assume each user possesses a one-dimensional value, denoted as $x$. Without loss of generality, we assume $x$ is within the normalized value range of $[-1,1]$. The subsequent three LDP mechanisms also adhere to this setting. Next, we define a randomization function that produces a perturbed value, denoted as $x^\prime = x + Lap(2/\epsilon)$, where $Lap(\alpha)$ represents a random variable following a Laplace distribution with probability density function $f(x)=\frac{1}{2\alpha}\mathrm{exp}(-\frac{|x|}{\alpha})$. To extend to $d\text{-dimensional}$ spaces, a straightforward approach is to apply the Laplacian mechanism independently to each dimension. Each dimension is allocated a privacy budget of $\epsilon/d$. According to the sequential composition property~\cite{dwork2008differential} of the LDP mechanism, it adheres to $\epsilon\text{-LDP}$. 

\subsubsection{Gaussian mechanism}
The Gaussian mechanism~\cite{balle2018improving} is another well-established LDP mechanism recognized for its effectiveness and flexibility. Consider a scenario where each user has a one-dimensional value $x$. We define a randomization function that produces a perturbed value denoted as $x=x+N(0,\sigma^2)$, where $N(0,\sigma^2)$ is a Gaussian random variable with zero mean and variance $\sigma^2$. In this context, the variance $\sigma^2$ is calibrated based on the privacy budget $\epsilon$ and the acceptable privacy leakage probability $\delta$. Specifically, we denote $\sigma^2$ as $\sigma^2=\frac{2\ln{(1.25/\delta)} }{\epsilon^2} $, ensuring that $(\epsilon,\delta)\text{-LDP}$ is satisfied, where $\delta>0$ serves as a relaxed privacy parameter. For multidimensional data, the Gaussian mechanism is applied independently to each dimension, mirroring the approach used in the Laplace mechanism. In this case, the privacy budget for each dimension is allocated as $\epsilon/d$. This independent application preserves the 
$(\epsilon,\delta)\text{-LDP}$ properties due to the sequential combination nature of local differential privacy, thus ensuring robust privacy protection across all dimensions.
\subsubsection{1-bit mechanism} In its one-dimensional form~\cite{ding2017collecting}, for any original data $x\in[-1, 1]$, the distribution followed by the perturbed value $x^\prime\in\{-1, 1\}$ is as follows:
\begin{equation}
	\mathrm{Pr}[x^\prime=c|x] =\begin{cases}
  \frac{1}{e^\epsilon +1}+\frac{x+1}{2}\cdot \frac{e^\epsilon -1}{e^\epsilon +1},  & \text{ if } c= 1\\
  \frac{e^\epsilon}{e^\epsilon +1}-\frac{x+1}{2}\cdot \frac{e^\epsilon -1}{e^\epsilon +1},& \text{ if } c=-1
\end{cases}.
\end{equation}
The extension to higher-dimensional spaces is implemented similarly to the first two mechanisms and is omitted here.

\subsubsection{Multi-bit Mechanism} 
The multi-bit mechanism~\cite{sajadmanesh2021locally,pei2023privacy} extends the 1-bit mechanism~\cite{ding2017collecting} to higher-dimensional spaces through a sampling process. Specifically, it operates by uniformly sampling $m$-dimensions from $d$-dimensional node feature data without replacement. Each of the $m$ selected dimensions is then perturbed using noise that satisfies $\epsilon/m\text{-LDP}$. Due to the sequential composition property of LDP mechanisms \cite{dwork2008differential}, the overall process adheres to $\epsilon\text{-LDP}$. After receiving all the perturbed data, the server transforms the reporting data $\hat{x}$ to its unbiased estimate $x^\prime =\frac{d}{m} \cdot \frac{e^{\epsilon/m} +1}{e^{\epsilon/m} -1} \cdot \hat{x}$. The variance of $x^\prime$ is as follows: $Var[x^\prime]=\frac{d}{m}\cdot \left ( \frac{e^{\epsilon/m} +1}{e^{\epsilon/m} -1} \right )^2-x^2$.

\section{Methodology}\label{S4}
In this section, we introduce PPGNN, a locally differentially private graph learning model designed to effectively meet users' personalized privacy requirements while significantly enhancing the utility of private graph learning. Next, we first introduce the overall framework of PPGNN in §\ref{sec4.1}, followed by a detailed discussion of the two key steps of PPGNN in §\ref{sec4.2} and~§\ref{sec4.3}, respectively. Finally, we conduct complexity analysis in §\ref{sec4.4} and privacy analysis in §\ref{sec4.5}.
\begin{figure*}
  \centering
  \includegraphics[scale=0.65]{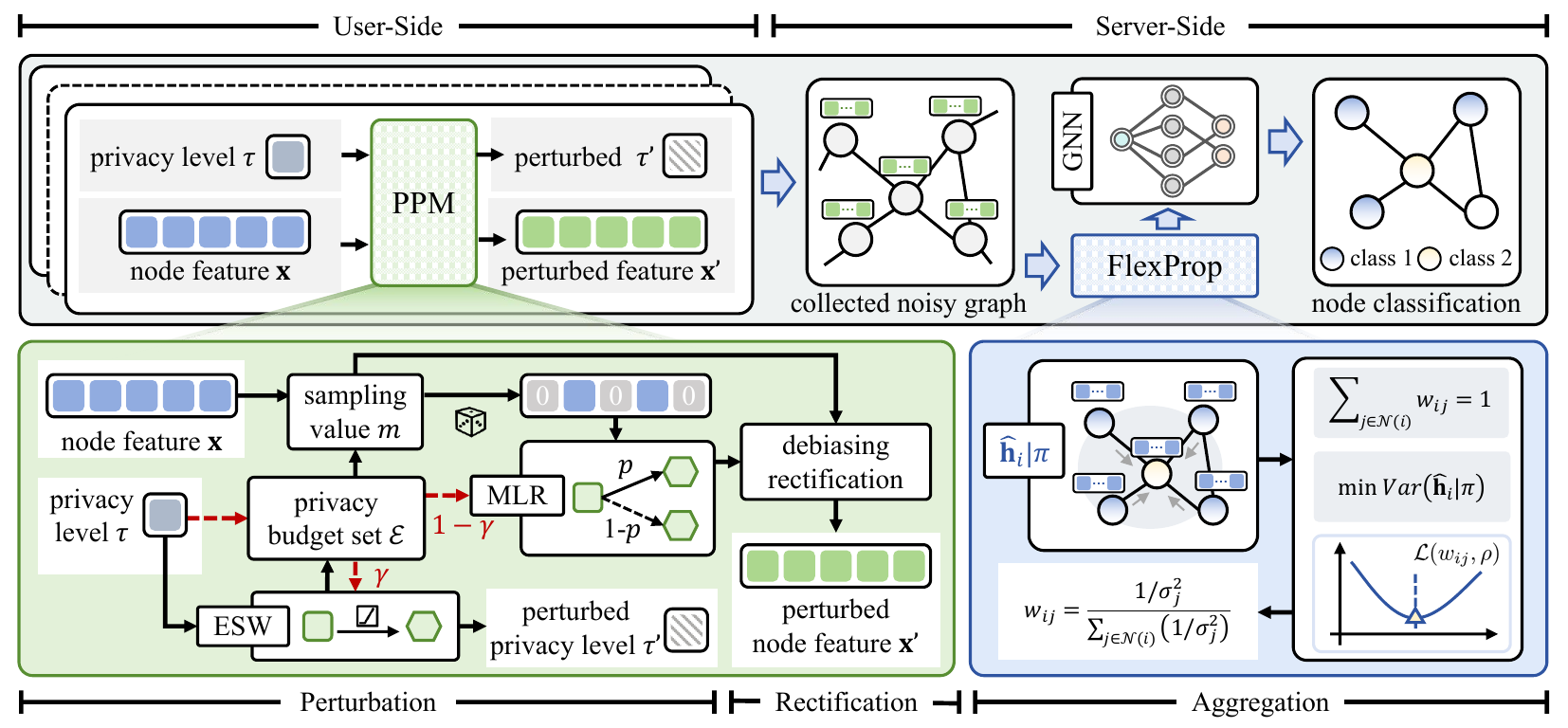}
  \vspace{-0.5em}
  \caption{Overview of our proposed PPGNN. Users run the Personalized Perturbation Mechanism (PPM) to their private data (node features $\mathbf{x}$ and privacy levels $\tau$) and send the output to the server. The server calibrates the collected noisy data using the FlexProp algorithm and performs GNN training for the node classification task. The red dashed path illustrates the allocation scheme for privacy budgets.}
  \label{fig:21}
  \vspace{-0.5em}
\end{figure*}

\vspace{-1em}
\subsection{Overview}\label{sec4.1}
In the scenario of locally private graph learning, the node features located on the user side are transmitted to the server after undergoing noise injection. The server performs private graph learning based on the collected node features. The overall framework of PPGNN is illustrated in Fig.~\ref{fig:21}. PPGNN addresses users' personalized privacy requirements effectively and enhances the utility of privacy graph learning through adaptive personalized aggregation. Specifically, in PPGNN, each user's local data consists of node features and privacy levels, which are protected using a \textit{Personalized Perturbation Mechanism (PPM)} that dynamically adjusts perturbation levels based on individual privacy preferences. On the server side, the \textit{FlexProp} algorithm is applied to correct the noisy node features, mitigating the biases introduced by the perturbation process. Once corrected, the server employs a GNN to train on the refined data, enabling downstream tasks such as node classification. PPGNN integrates personalized privacy protection with adaptive feature correction, achieving a balance between robust privacy guarantees and high utility, thus optimizing the performance of privacy-preserving graph learning. In summary, PPGNN consists of the following two parts:

\begin{algorithm}[t]
    \caption{Local Randomizer}
    \label{alg1}
    \textbf{Input}: Numeric data $x \in [-1,1]$, privacy budget
 $\epsilon$.\\
    \textbf{Output}: Randomized response $x^\prime\in\left \{ -\frac{e^{\epsilon}+1}{e^{\epsilon}-1}, \frac{e^{\epsilon}+1}{e^{\epsilon}-1}\right \}$.
    
    \begin{algorithmic}[1] 
        \STATE Sample a Bernoulli variable $u$ such that\\ $\text{Pr}[u=1]= \frac{e^{\epsilon}-1}{2e^{\epsilon}+2}\cdot x+\frac{1}{2} $.
        \IF{$u=1$}
            \STATE $x^\prime=(e^{\epsilon}+1)/(e^{\epsilon}-1)$.
            \ELSE
             \STATE $x^\prime=-(e^{\epsilon}+1)/(e^{\epsilon}-1)$.
             \ENDIF
        \STATE \textbf{return} $x^\prime$ 
    \end{algorithmic}
\end{algorithm}

\subsubsection{Perturbing node features with PPM}
Although several LDP mechanisms applicable to node features have been proposed (as outlined in Section~\ref{3.1}), these mechanisms fail to address the dual privacy requirements inherent in personalized privacy graph learning scenarios. Additionally, there remains significant potential to improve the utility of these mechanisms by reducing noise perturbation variance to enhance data accuracy. Balancing these challenges is non-trivial, as it requires simultaneously minimizing noise and preserving dual privacy. To tackle this issue, we introduce a novel LDP perturbation algorithm, the Personalized Perturbation Mechanism (PPM), which incorporates two interdependent modules: a multi-dimensional local randomizer and an extended square-wave mechanism. These modules work synergistically to effectively reduce noise variance while addressing dual privacy concerns. In the PPGNN framework, each user perturbs their local data using PPM before transmitting the processed data to the server for further private graph learning.

\subsubsection{Calibrating noise features with FlexProp}
Directly utilizing the noisy data generated by the LDP mechanism for graph learning often leads to poor utility, necessitating further calibration to reduce noise. To address this, we employ multi-hop aggregation, which is theoretically effective in minimizing estimation errors. However, traditional multi-hop aggregation is suboptimal in personalized privacy-preserving scenarios, as it fails to leverage the effective information of each node fully. To overcome this limitation, we propose the FlexProp algorithm, which efficiently minimizes estimation errors through adaptive weighted aggregation, tailoring the process to individual nodes' privacy requirements and data features.

\vspace{-1em}
\subsection{Perturbing Node Features with PPM}\label{sec4.2}
To achieve personalized privacy preservation for node features, we introduce a novel LDP algorithm called the \textit{Personalized Perturbation Mechanism (PPM)}, as outlined in Alg.~\ref{alg3}. The architectural design of PPM is illustrated in Fig.~\ref{fig:21}. In this framework, we strategically divide the privacy budget into $h$ discrete levels, denoted as 
$\mathcal{E}=\cup_{\tau=1 }^{h}\left\{\epsilon^\tau\right\}$. Here, $\Delta_\epsilon =\epsilon ^{\tau+1}-\epsilon ^\tau>0$ indicates that each subsequent level represents a progressively greater privacy budget, while $\epsilon^\tau$ denotes the specific budget at privacy level $\tau$. Notably, a smaller value of $\tau$ corresponds to a reduced privacy budget, leading to the injection of a greater amount of noise. As depicted in Fig.~\ref{fig:21}, the PPM consists of two critical components: the \textit{Multi-dimensional Local Randomizer (MLR)} and the \textit{Extended Square Wave (ESW)} mechanism. The allocation of the privacy budget between these two components is governed by the parameter $\gamma$. Specifically, the ESW utilizes $\gamma\epsilon^\tau$ of the privacy budget, while the MLR is allocated $(1-\gamma)\epsilon^\tau$. This careful distribution ensures that each component effectively contributes to the overall privacy protection mechanism, balancing the trade-off between noise injection and data utility. Next, we explore the details of MLR and ESW respectively.

\subsubsection{Multi-dimensional Local Randomizer} 
To protect multi-dimensional node features, we introduce the MLR algorithm, as presented in Alg.~\ref{alg2}. This algorithm extends the concept of the \textit{local randomizer} (Alg.\ref{alg1})~\cite{duchi2018minimax}, originally designed for unidimensional spaces, to accommodate multidimensional contexts. The MLR algorithm\footnote{The MLR algorithm is not limited to the local randomizer in~\cite{duchi2018minimax}; it serves as a general wrapper that can integrate any one-dimensional fixed-budget perturbation mechanism. This modularity allows MLR to be adapted to diverse privacy mechanisms (\textit{e.g.}, Laplace).} effectively allocates the total privacy budget across $m$ attributes, rather than the original $d$, which significantly reduces the noise variance associated with the perturbation process. As a trade-off, the additional estimation error introduced by sampling $m$ attributes from $d$ can be managed by judiciously selecting an appropriate $m$ value. This balancing act is formally addressed in Theorem~\ref{thm2}. Furthermore, Theorem~\ref{thm3} establishes that the MLR algorithm satisfies $\epsilon$-LDP and proves that $\mathbb{E}\left[\mathbf{x}^\prime\right]=\mathbf{x}$.

\begin{algorithm}[tb]
    \caption{Multi-dimensional Local Randomizer}
    \label{alg2}
    \textbf{Input}: Feature vector $\mathbf{x} \in [-1,1]^{d}$, privacy budget
 $\epsilon$\\
    \textbf{Output}: $\mathbf{x}^\prime\in\left \{ -\frac{d}{m}\cdot\frac{e^{\epsilon/m}+1}{e^{\epsilon/m}-1}, 0,  \frac{d}{m}\cdot\frac{e^{\epsilon/m}+1}{e^{\epsilon/m}-1}\right \} ^d$
    
    \begin{algorithmic}[1] 
        \STATE Let $\mathbf{x}^\prime=\left \langle 0,0,\dots,0  \right \rangle $.
        \STATE Let $m=\max(1,\min (d,\left \lfloor \epsilon/2.2  \right \rfloor)) $.
         \STATE Sample $m$ values uniformly without replacement from $\{1,2,\dots,d\}$.
        \FOR{each sampled value $j$}
            \STATE Feed $\mathbf{x}_j$ and $\epsilon/m$ to the Alg.~\ref{alg1} to obtain $t_j$.
            \STATE $\mathbf{x}_j^\prime = \frac{d}{m}\cdot t_j$.
        \ENDFOR
        \STATE \textbf{return} $\mathbf{x}^\prime$    
    \end{algorithmic}
\end{algorithm}

\begin{theorem}\label{thm2}
To minimize the variance of the perturbation mechanism (as given by Equation (\ref{5})), the optimal sampling parameter $m=\max(1,\min (d,\left \lfloor \epsilon/2.2  \right \rfloor)) $.
    \begin{equation}
		Var\left [ \mathbf{x}_i^\prime \right ]=\frac{d}{m}\cdot\left (  \frac{e^{\epsilon/m}+1}{e^{\epsilon/m}-1} \right )^2-\mathbf{x}_i^2.\label{5}
	\end{equation}
\end{theorem}
\begin{proof}

Given that variance is a critical factor that affects the accuracy of the estimation of the LDP mechanisms: the smaller the variance, the more accurate the estimation, the theorem selects the optimal value of $m$ by minimizing the variance of the LDP mechanism. See Appendix A for more details.
\end{proof}

\begin{theorem}\label{thm3}
    Algorithm~\ref{alg2} (MLR)
satisfies $\epsilon$-LDP for each node. In addition, for any $i\in\{1,2,\dots,d\}$, $\mathbb{E}\left[\mathbf{x}^\prime_i\right]=\mathbf{x}_i$.
\end{theorem}
\begin{proof}
We denote the MLR algorithm as $\mathcal{R}^1$. According to Definition~\ref{def:1}, our objective is to demonstrate that for any two input features $\mathbf{x}_1$ and $\mathbf{x}_2$ satisfying the following inequality: $\Pr\left[\mathcal{R}^1(\mathbf{x}_1) = \mathbf{x}^{\prime}\right] \le e^{\epsilon} \cdot\Pr\left[\mathcal{R}^1(\mathbf{x}_2) = \mathbf{x}^{\prime}\right]$, where $\mathbf{x}^{\prime}$ represents a perturbed output. Furthermore, ensuring that the data are unbiased (that is, $\mathbb{E}\left[\mathbf{x}^\prime_i\right]=\mathbf{x}_i$) guarantees the statistical utility of the data. We have also provided a proof of this property. For further details, please refer to the Appendix~B.
\end{proof}

Note that the user's privacy level $\tau$ is also sensitive and utilized in the subsequent FlexProp algorithm. Directly applying the MLR algorithm to protect node features may potentially reveal $\tau$, as $\tau$ determines $\epsilon^\tau$, which in turn affects the final output. To address this issue, we propose the Extended Square Wave (ESW) mechanism. Additionally, considering that $m$ could potentially leak the user's privacy level, instead of directly using $\tau$, we use a randomized version $\tau^\prime$ through the ESW, that is: $m^\prime=\max(1,\min (d,\left \lfloor \epsilon^{\tau^\prime
}/2.2\right\rfloor))$.
\subsubsection{Extended Square Wave mechanism} To protect the users' privacy levels, the ESW extends from the square wave mechanism~\cite{li2020estimating} and can be applied to the protection of discrete domain data after extension. In ESW, the input domain is defined as $\mathcal{D} = \{1, 2, \dots, f\}$, where $f$ is the total number of privacy levels, and the input value $g \in \mathcal{D}$ represents a user’s true privacy level index. The output domain is extended to $\mathcal{D}^\prime = \{1, 2, \dots, f+2b\}$, where $g^\prime \in \mathcal{D}^\prime$ is the perturbed output value. The parameter $b$ controls the width of the high-probability window around $g$ and is defined as: $b = \left \lfloor \frac{\epsilon e^\epsilon - e^\epsilon + 1}{2e^\epsilon (e^\epsilon - 1 - \epsilon)} \cdot f \right \rfloor$. This extension allows the mechanism to assign higher output probability near the true value to preserve utility, while still satisfying $\epsilon$-LDP through randomization across a larger space. The perturbation process is as follows:
\begin{equation}\label{eq1}
    \mathrm{Pr} [\mathrm{ESW}(g)=g^\prime]=\begin{cases}p,&\text{if } |g-g^\prime |\le b\\
  q, &\text{otherwise}
\end{cases},
\end{equation}
where $p=e^\epsilon/({(2b+1)e^\epsilon+f-1})$, $q=1/((2b+1)e^\epsilon+f-1)$. The values of $p$ and $q$ ensure $\epsilon$-local differential privacy over the extended output domain $\mathcal{D}'$, while the use of the high-probability window improves the fidelity of perturbed outputs.

\vspace{-1em}
\subsection{Calibrating Noise Features with FlexProp}\label{sec4.3}
In this section, we first conduct an error analysis to identify the key factors influencing the utility enhancement of perturbed node features. Building on these insights, we introduce FlexProp, a weighted aggregation algorithm specifically designed for personalized scenarios, which efficiently calibrates the perturbed node features to maximize utility while respecting individual privacy requirements.
\subsubsection{Error analysis}
The server proceeds with GNN training after receiving all the perturbed node features. During the first layer of GNN, the embedding for any given node $v\in\mathcal{V}$ is produced through the following process:
\begin{equation}
\widehat{\mathbf{h}}_{\mathcal{N}(v)}= \textsc{Agg}\left(\{\mathbf{x}_u^\prime, \forall u \in \mathcal{N}(v)\}\right),\label{88}
\end{equation}
where $\widehat{\mathbf{h}}_{\mathcal{N}(v)}$ is the estimation of the $\textsc{Agg}(\cdot)$ function that aggregates the feature vectors $\mathbf{x}_u^\prime$ of all perturbed nodes $u \in \mathcal{N}(v)$. The aggregation process provides an unbiased estimate when $\mathbb{E}[\mathbf{x}^\prime_u]=\mathbf{x} _u$ (Theorem~\ref{thm3}) and the $\textsc{Agg}$ function operates linearly,  as established in Theorem~\ref{thm1}.

\begin{theorem}\label{thm1}
The linear aggregation function described by Eq.~\eqref{88} guarantees unbiased estimation. Specifically, for any node $v\in\mathcal{V}$, it holds that: $\mathbb{E}[\widehat{\mathbf{h}}_{\mathcal{N}(v)}] = \mathbf{h}_{\mathcal{N}(v)}$.
\end{theorem}
\begin{proof}
See Appendix~C for the proof.
\end{proof}
The estimation error $\Delta$ between the estimated node embedding $\widehat{\mathbf{h}}_{\mathcal{N}(v)}$ and the true node embedding $\mathbf{h}_{\mathcal{N}(v)}$ for any node $v$ is defined as the maximum absolute difference across each dimension $i$ in the embedding, expressed as:
\begin{equation}
    \Delta=\max_{i\in\{1,2,\dots,d\}} |( \widehat{\mathbf{h}}_{\mathcal{N}(v)} )_i-( \mathbf{h}_{\mathcal{N}(v)})_i|.
\end{equation}
where $( \widehat{\mathbf{h}}_{\mathcal{N}(v)} )_i$ and $( \widehat{\mathbf{h}}_{\mathcal{N}(v)} )_i$ represent the $i$-th dimension of the estimated and true node embeddings, respectively. This error quantifies the deviation introduced by the perturbation and subsequent aggregation process. The detailed analysis of $\Delta$ is provided in Theorem~\ref{thm4}.
\begin{theorem}\label{thm4}
    Let $\textsc{Agg}(\cdot)$ be the \textsc{Mean} aggregator and $\delta>0$, with at least $1-\delta$ probability, the estimation error $\Delta$ for any given node $v$ is: $\Delta = \mathcal{O}\left(\frac{\sqrt{d \log(d/\delta)}}{(1-\gamma)\epsilon \sqrt{|\mathcal{N}(v)|}}\right)+\mathcal{O}\left(\frac{1}{\gamma\epsilon}\right)$.
\end{theorem}
\begin{proof}
See Appendix~D for the proof.
\end{proof}
The above theorem indicates that, given a fixed $\epsilon$, selecting an appropriate $\gamma$ and expanding the size of $\mathcal{N}(v)$ are both beneficial in reducing the estimation error. The trade-offs of $\gamma$ are discussed in Section~\ref{gamma}.

\subsubsection{FlexProp algorithm}
To further minimize the error, we introduce a denoising aggregation algorithm called FlexProp, which calibrates noise by aggregating larger neighborhoods. Previous works~\cite{morris2019weisfeiler,abu2019mixhop,sajadmanesh2021locally,huang2024higher} have shown that considering higher-order neighborhoods aids in learning superior node representations. As illustrated in Alg.~\ref{alg4}, for a node $v$, FlexProp aggregates features of nodes up to the farthest $K$ steps away from $v$ by invoking the $\textsc{Agg}$ function $K$ consecutive times, without any non-linear transformation in between, which facilitates noise calibration. Additionally, FlexProp examines the influence of aggregation \textit{policy} $\pi$ on enhancing aggregation accuracy. In a personalized privacy-preserving context, varying aggregation policies can substantially impact accuracy due to differing noise scales among neighboring nodes. We define the estimated feature of node $i$ post-aggregation under policy $\pi$ as: $\widehat{\mathbf{h}}_i(\pi)= \textsc{Agg}\left(\{\mathbf{x}_j^\prime, \forall j \in \mathcal{N}(i)\}\right)$. Assuming $w_{ij}$ represents the weight of node $j$ when aggregated to target node $i$, and considering a weighted aggregation function, we have:
\begin{equation}
\widehat{\textbf{h}}_i(\pi)= \frac{1}{ {\textstyle \sum_{j\in\mathcal{N} (i)}w_{ij}}  }  {\textstyle \sum_{j\in\mathcal{N} (i)}w_{ij}\mathbf{x}_j^\prime}.
\end{equation}
Furthermore, assuming the feature vector $\mathbf{x}^\prime_i$ of each node $i$ decomposes into the true signal $\mathbf{s}_i$ and noise $\mathbf{n}_i$, i.e., $\mathbf{x}^\prime_i=\mathbf{s}_i+\mathbf{n}_i$, where $\mathbf{n}_i$ is a random variable with a mean of 0 and a variance of $\sigma^2_i$, we have: $\widehat{\textbf{h}}_i(\pi)= \frac{1}{ {\textstyle \sum_{j\in\mathcal{N} (i)}w_{ij}}  }  {\textstyle \sum_{j\in\mathcal{N} (i)}w_{ij}\left ( \mathbf{s}_j +\mathbf{n} _j\right ) }$. Next, we analyze the variance of the noise after aggregation, denoted as $Var  ( \widehat{\textbf{h}}_i|\pi)$, to better assess the denoising effectiveness and determine the optimal aggregation policy. The variance is calculated as follows:
\begin{equation}
Var \left ( \widehat{\textbf{h}}_i|\pi \right )  = \frac{1}{ \left ( {\textstyle \sum_{j\in\mathcal{N} (v)}w_{ij}}  \right )^2  }  {\textstyle \sum_{j\in\mathcal{N} (v)}w_{ij}^2\sigma_j^2  }.
\end{equation}
To achieve optimal noise removal, the variance of the aggregated noise should be minimized. We aim to select a set of weights $w_{ij}$ that satisfy the constraint $ {\textstyle \sum_{j\in\mathcal{N}(i)}} w_{ij}=1$ while minimizing this variance. To find the optimal solution, we employ the \textit{Lagrange multiplier} method~\cite{bertsekas2014constrained}. The Lagrange function is defined as:
\begin{equation}
 \mathcal{L} (w_{ij},\rho)=\sum_{j\in \mathcal{N} (i)}w_{ij}^2\sigma^2_j+ \rho \left ( \sum_{j\in\mathcal{N} (i)}w_{ij}-1  \right ).   
\end{equation}

\begin{algorithm}[tb]
    \caption{Personalized Perturbation Mechanism}
    \label{alg3}
    \textbf{Input}: User's feature vector $\mathbf{x} \in [-1,1]^{d}$, privacy level $\tau\in\{1,2,\dots,h\}$, privacy budget set $\mathcal{E}=\cup_{\tau=1 }^{h}\left\{\epsilon^\tau\right\}$,  allocation parameter $\gamma$\\
    \textbf{Output}: Perturbed node feature $\mathbf{x}^\prime$ and $\tau^\prime$
    
    \begin{algorithmic}[1] 
        \STATE $\epsilon^\tau\leftarrow \mathcal{E}(\tau)$. \Comment{total privacy budget}
        \STATE $\epsilon_1\leftarrow \gamma\epsilon^\tau$. \Comment{privacy budget of the privacy level}
        \STATE $\epsilon_2\leftarrow (1-\gamma)\epsilon^\tau$. \Comment{privacy budget of the node feature vector}
        \STATE $\tau^\prime\leftarrow \mathrm{ESW}(\tau, \epsilon_1)$ and $\mathbf{x}^\prime\leftarrow \mathrm{MLR}(\mathbf{x}, \epsilon_2)$.
        \STATE \textbf{return} $\mathbf{x}^\prime$ and $\tau^\prime$     
    \end{algorithmic}
\end{algorithm}

The optimal weights are obtained by solving the above equation as $w_{ij}=\frac{1/\sigma_j^2}{ {\textstyle \sum_{j\in \mathcal{N} (i)}}\left ( 1/\sigma_j^2 \right ) }$. For more details, please refer to the Appendix~E. 

In practice, the optimal weights under the given policy $\pi$ are not always feasible. To address this, we approximate the optimal weights using the neighbor aggregation function $\mathbf{\Pi}$ and incorporate it into Alg.~\ref{alg4}.\footnote{Unlike the KProp Layer in~\cite{sajadmanesh2021locally}, which performs multi-hop aggregation under a global privacy budget, our FlexProp is tailored for personalized LDP scenarios, where users have heterogeneous noise levels. FlexProp introduces a privacy-aware weight matrix $\Pi$ to calibrate the aggregation process while preserving privacy, which is not considered in KProp.} The function $\mathbf{\Pi}$ can be represented in matrix form as follows, 
\begin{equation}
\mathbf{\Pi}_{i,j}=\begin{cases}
   0,&\text{ if } \textbf{A}_{i,j}=0 \\
   \alpha_{i,j} ,  &\text{ if } \textbf{A}_{i,j}=1
\end{cases},\label{14}
\end{equation}
where $\mathbf{A}$ is the adjacency matrix and $\alpha_{i,j}=\tau^\prime _j/ {\textstyle \sum_{j\in\mathcal{N}(i)}\tau^\prime_j} $.

\vspace{-1em}
\subsection{Complexity Analysis}\label{sec4.4}
The time complexity of PPGNN is determined by its two key components: the PPM and the FlexProp aggregation algorithm. PPM perturbs the privacy level and feature vector of each node locally, requiring $\mathcal{O}(|\mathcal{V}|\cdot d)$, where $|\mathcal{V}|$ is the number of nodes and $d$ is the feature dimension. FlexProp performs $K$ layer message aggregation, involving $\mathcal{O}(K\cdot E \cdot d)$, where $E$ is the number of edges and $K$ is the number of layers. Combining these, the overall time complexity is $\mathcal{O}(|\mathcal{V}|\cdot d+K\cdot E \cdot d)$, scaling linearly with graph size and feature dimensionality. This linear scalability ensures computational efficiency in practice and makes PPGNN suitable for real-world applications involving large, high-dimensional graphs.

\begin{algorithm}[tb]
    \caption{FlexProp}
    \label{alg4}
    \textbf{Input}: Graph $\mathcal{G}=(\mathcal{V}, \mathbf{A})$, privacy level $\tau_v, \forall v\in\mathcal{V}$, input features $\mathbf{x}_v$, aggregation function \textsc{Agg}, propagation step $K\ge0$\\
    \textbf{Output}: Final embedding vector $\mathbf{h}_v, \forall v \in \mathcal{V}$

    \begin{algorithmic}[1] 
        \STATE Compute propagation weights $\mathbf{\Pi}$ using Eq.~\eqref{14}
        \FOR{\textbf{each} node $v \in \mathcal{V}$}
            \STATE Initialize hidden representation: $\mathbf{h}_v^0 = \mathbf{x}_v$
        \ENDFOR
        \FOR{$k = 1$ \textbf{to} $K$}
            \FOR{\textbf{each} node $v \in \mathcal{V}$}
                \STATE $\mathbf{h}_v^k = \textsc{Agg}\left( \left\{ \mathbf{\Pi}_{v,u} \cdot \mathbf{h}_u^{k-1} \mid u \in \mathcal{N}(v) \right\} \right)$
            \ENDFOR
        \ENDFOR
        \STATE \textbf{return} $\mathbf{h}_v^K, \forall v \in \mathcal{V}$
    \end{algorithmic}
\end{algorithm}
\subsection{Privacy Analysis}\label{sec4.5}
Recalling the definition of LDP (Definition~\ref{def:1}), the PPGNN satisfies  $\epsilon^\tau$-LDP as established by Theorem~\ref{thm5}, which relies on the composition property (Proposition~\ref{pos1}) and robustness to post-processing (Proposition~\ref{pos2}) in differential privacy~\cite{dwork2014algorithmic}. Furthermore, any subsequent predictions made by the PPGNN are governed by the post-processing theorem, since the LDP mechanism is applied only once to the private data.
\begin{theorem}\label{thm5}
   PPGNN satisfies $\epsilon^\tau$-LDP for each node.
\end{theorem}
\begin{proof}
See Appendix~F for the proof.
\end{proof}

\section{Experiments}\label{S5}

\subsection{Experimental Settings}
\subsubsection{Datasets}
Our experiments use six publicly real-world datasets: Cora~\cite{yang2016revisiting}, CiteSeer~\cite{yang2016revisiting}, Pubmed~\cite{yang2016revisiting}, LastFM~\cite{DBLP:conf/cikm/RozemberczkiS20},  Facebook~\cite{rozemberczki2021multi}, and Wikipedia~\cite{rozemberczki2021multi}. The statistical details of these datasets are presented in Table~\ref{tab1}.
{\begin{itemize}
\item \textbf{Cora}~\cite{yang2016revisiting}: A citation network in which each node corresponds to a scientific paper and edges signify citation links between papers. Node features are derived from paper content, and each node is labeled with a research topic.
\item \textbf{CiteSeer}~\cite{yang2016revisiting}: Also a citation network where nodes represent academic papers and edges denote citation relationships. Each node has a content-based feature vector and belongs to a research category.
\item \textbf{Pubmed}~\cite{yang2016revisiting}: A large-scale citation network in which nodes are biomedical papers and edges represent citations. Node features are TF/IDF weighted term vectors, and the task is multi-class classification among diabetes-related topics.
\item \textbf{LastFM}~\cite{DBLP:conf/cikm/RozemberczkiS20}: A social network where nodes correspond to LastFM music platform users and edges reflect friendship connections. The task is to predict each user's home country based on the artists they have liked.
\item \textbf{Facebook}~\cite{rozemberczki2021multi}: A social network in which each node is an official Facebook page and edges represent mutual ``likes'' between pages. Node features are derived from page descriptions, and the task is page-type classification.
\item \textbf{Wikipedia}~\cite{rozemberczki2021multi}: A co-occurrence network constructed from Wikipedia pages, where nodes represent articles and edges connect pages that share common links. The task is binary classification of page category.
\end{itemize}}
\begin{table}
	\centering
	\small
	\sc
 \caption{Statistics of graph datasets}
	\setlength{\tabcolsep}{1.8mm}\begin{tabular}{l|rrrr}
		\toprule
		Dataset  & \ \#Classes & \ \#Nodes & \ \#Edges & \ \#Features   \\
		\midrule
		Cora     & 7         & 2,708   & 5,278   & 1,433      \\
		Citeseer & 6  & 3,327  & 4,552   & 3,703  \\
        Pubmed     & 3         & 19,717   & 44,324   & 500      \\
		LastFM   & 18         & 7,624  & 27,806  & 7,842       \\
		Facebook & 4         & 22,470  & 170,912 & 4,714   \\
				Wikipedia   & 2  & 11,631 & 170,845 & 13,183 \\
		\bottomrule
	\end{tabular}
 \label{tab1}
\end{table}
\begin{figure*}[h!]
	\centering
	\begin{tabular}{cccc}
		\multicolumn{4}{c}{\textbf{}} \\
    \includegraphics[width=0.23\linewidth]{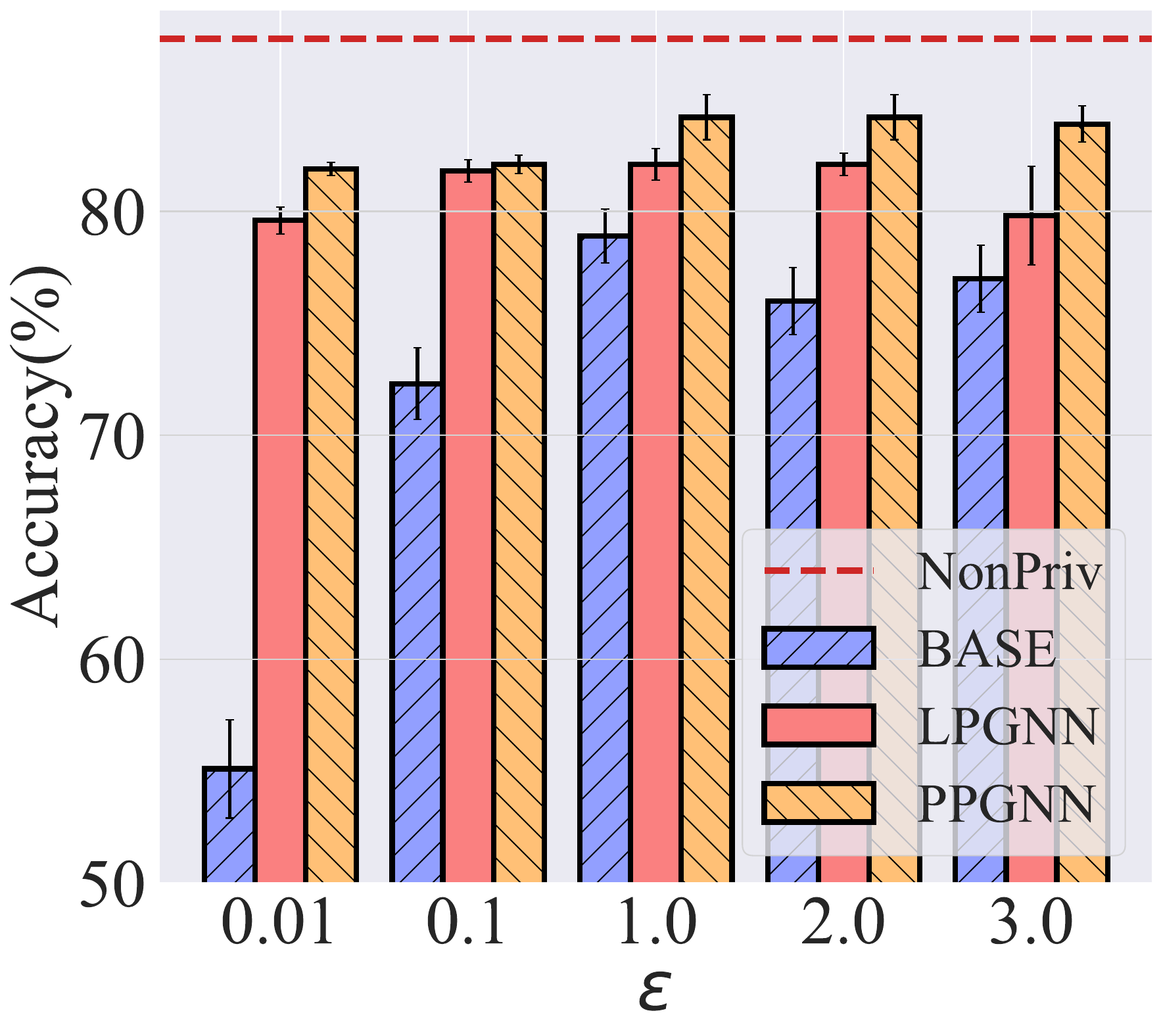} & 
	\includegraphics[width=0.23\linewidth]{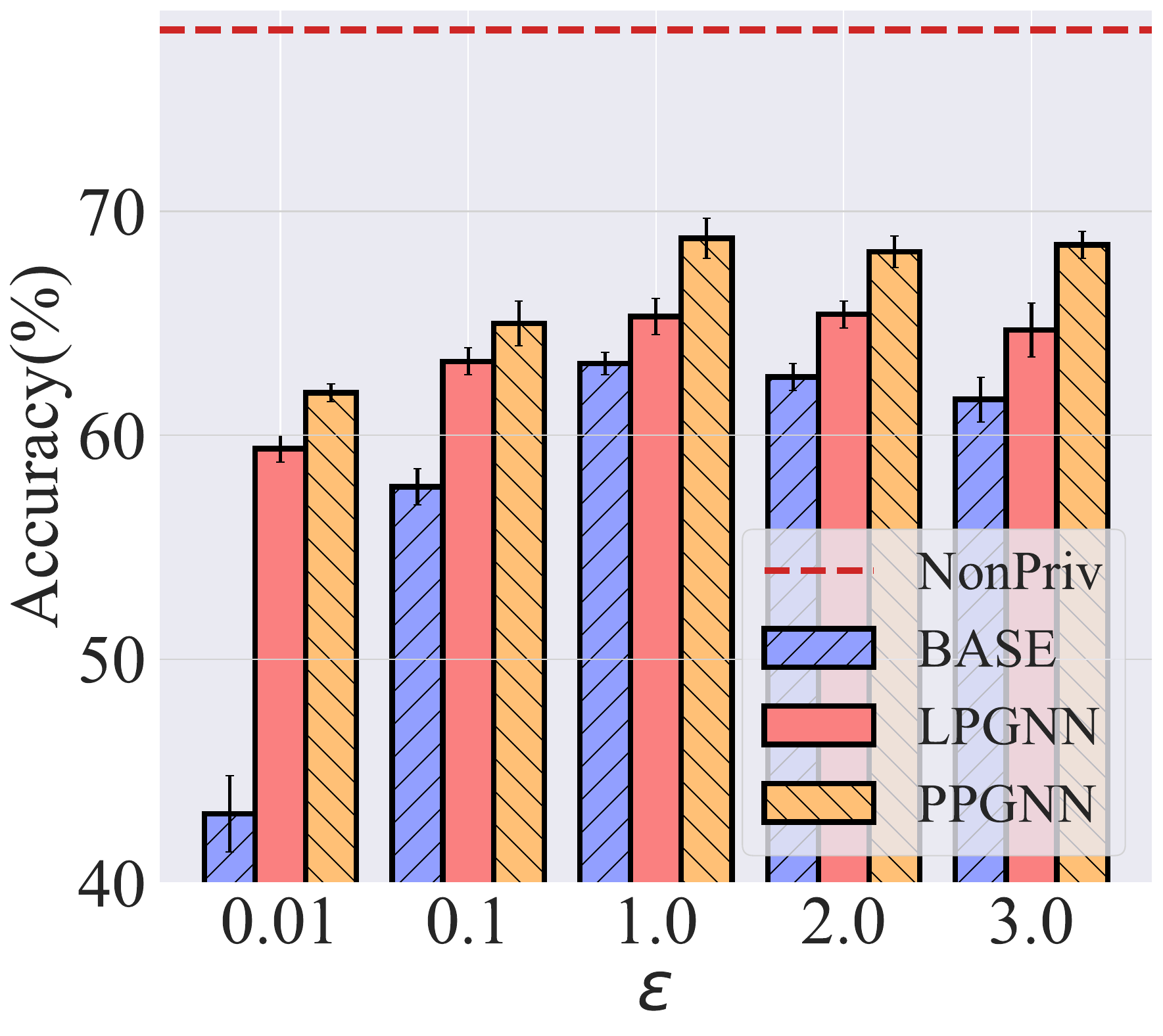} & 
	\includegraphics[width=0.23\linewidth]{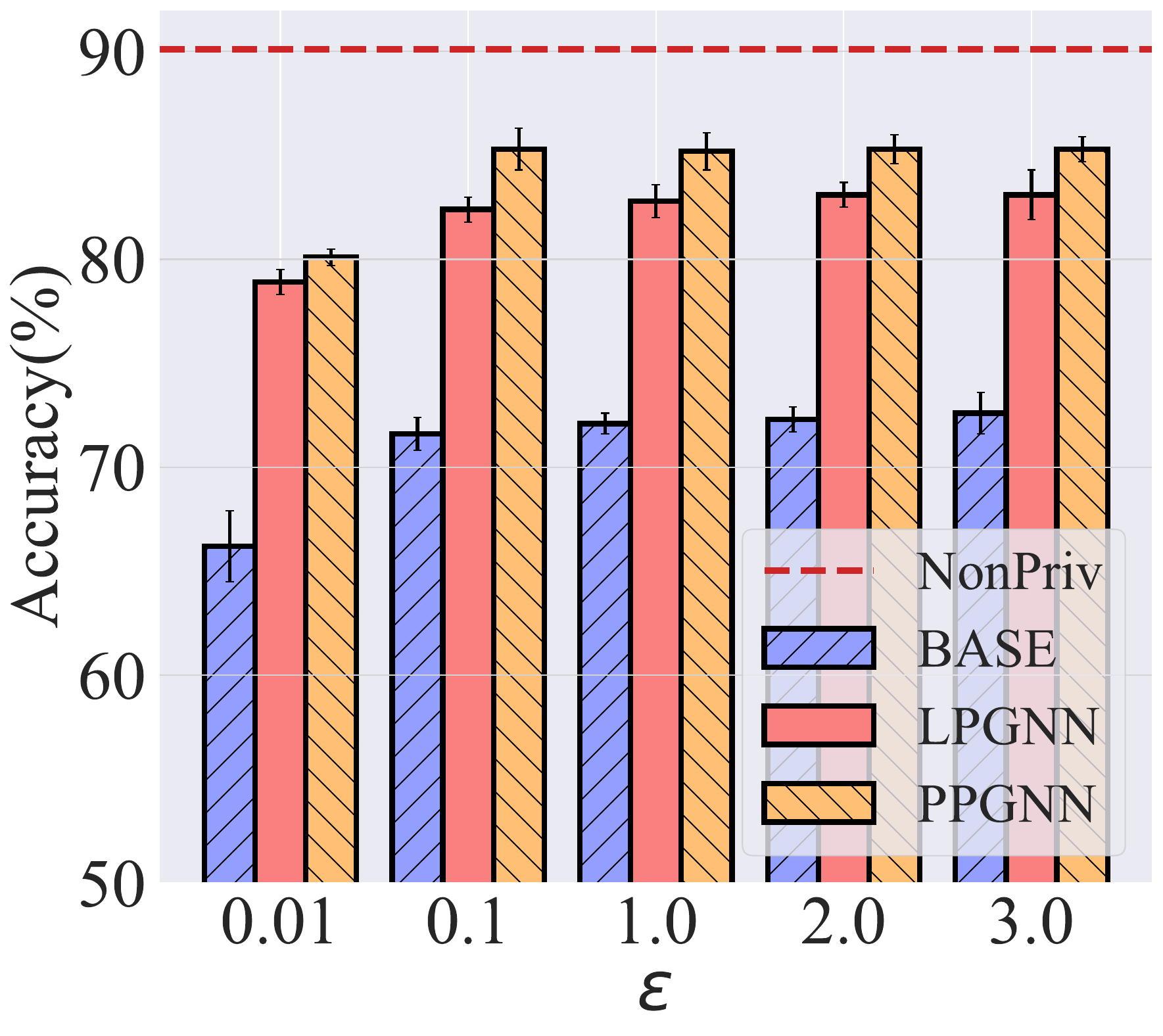} & 
	\includegraphics[width=0.23\linewidth]{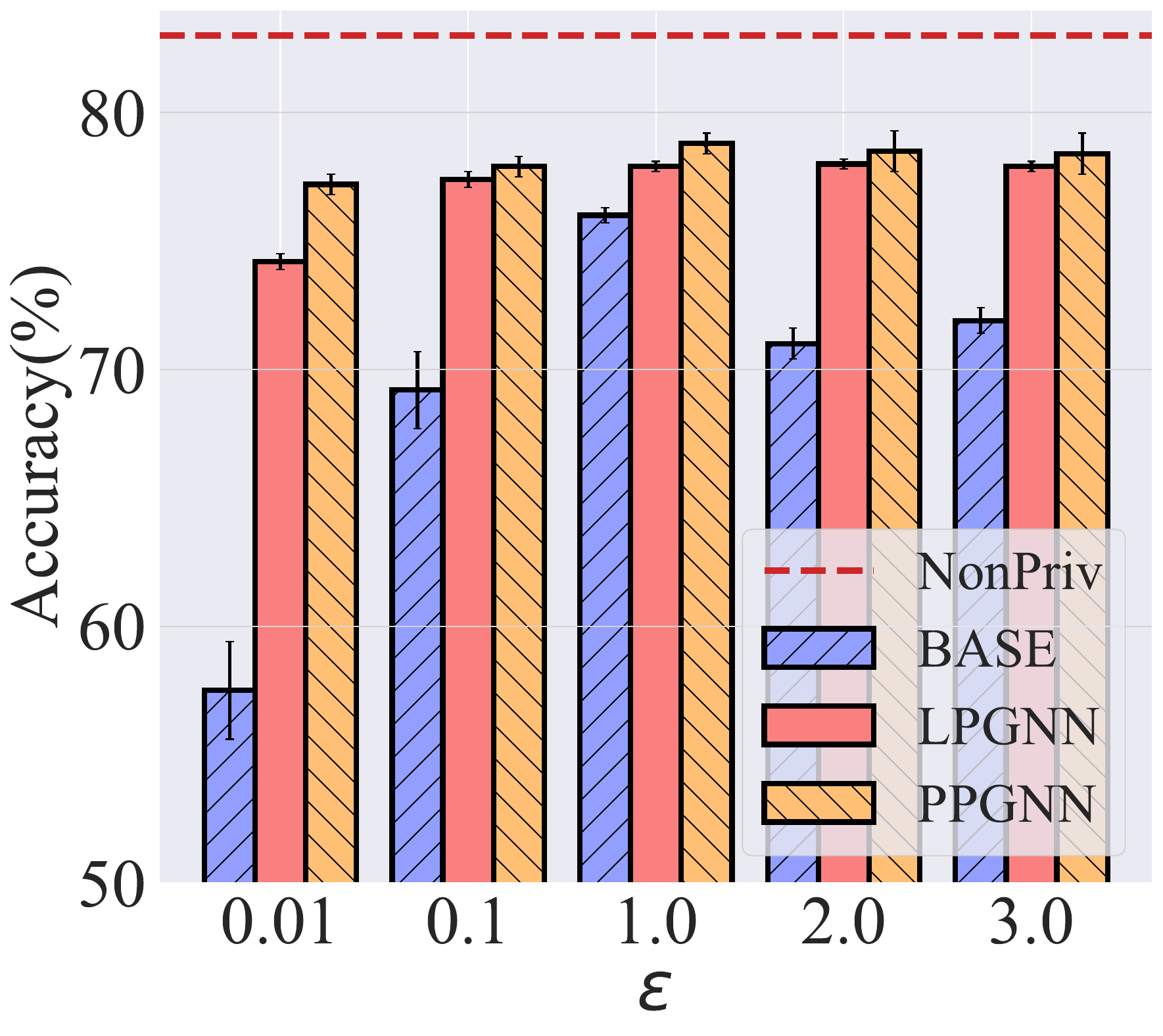}  \\
  \textbf{(a) Cora} (GraphSAGE) & \textbf{(b) CiteSeer} (GraphSAGE) & \textbf{(c) Pubmed} (GraphSAGE) &\textbf{(e) LastFM} (GraphSAGE) \\
  \multicolumn{4}{c}{\textbf{}} \\
    \includegraphics[width=0.23\linewidth]{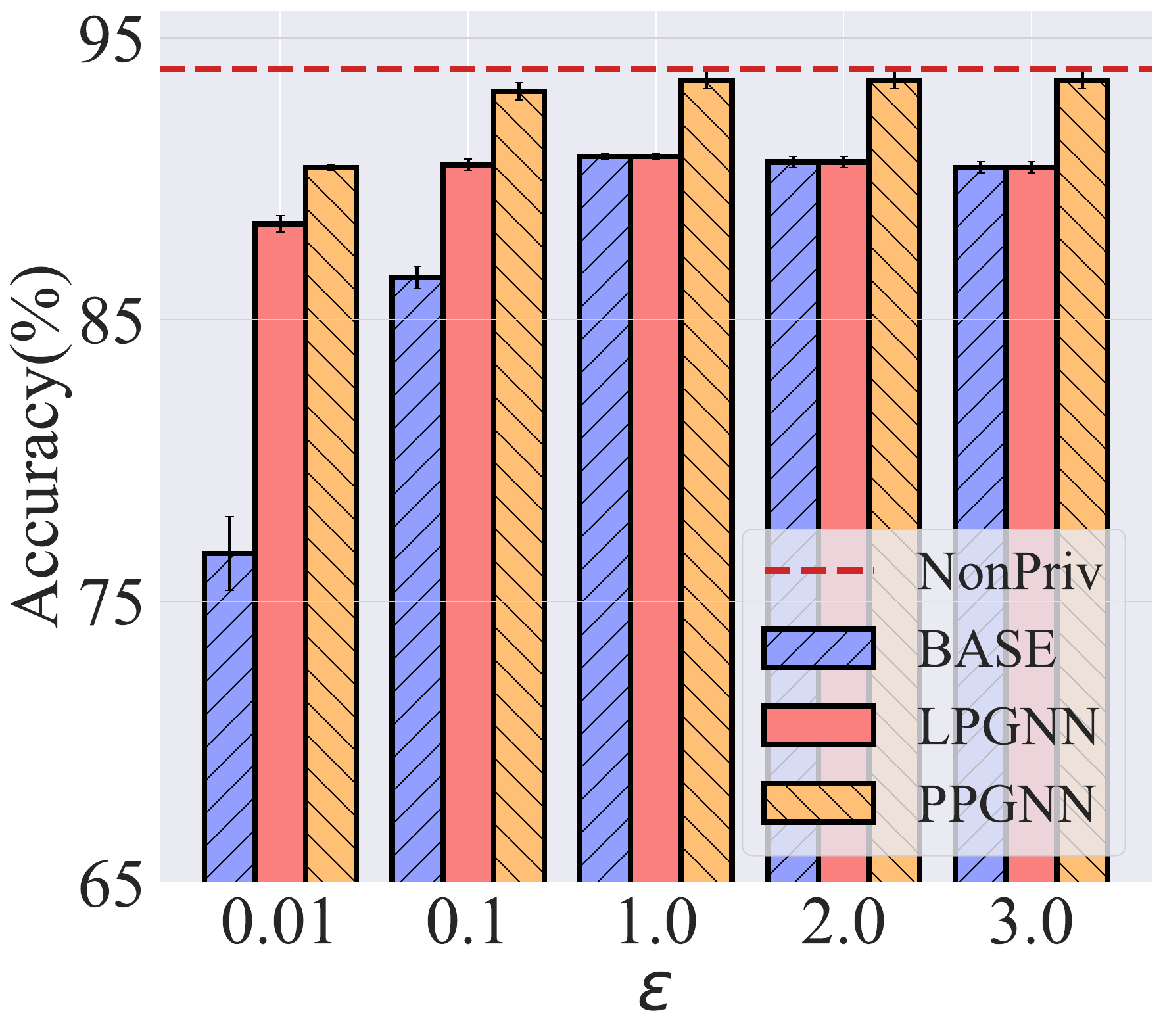} & 
	\includegraphics[width=0.23\linewidth]{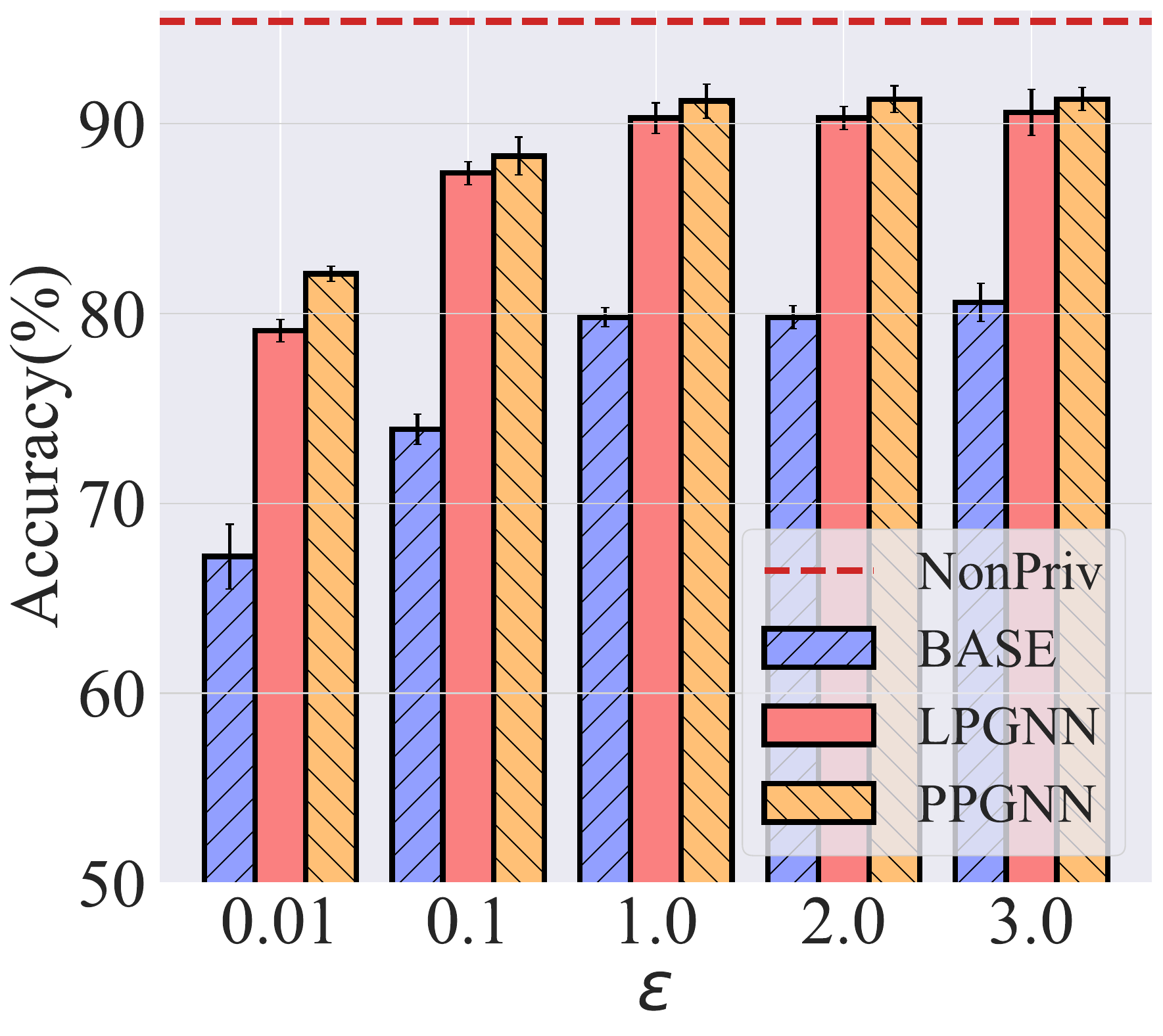} & 
	\includegraphics[width=0.23\linewidth]{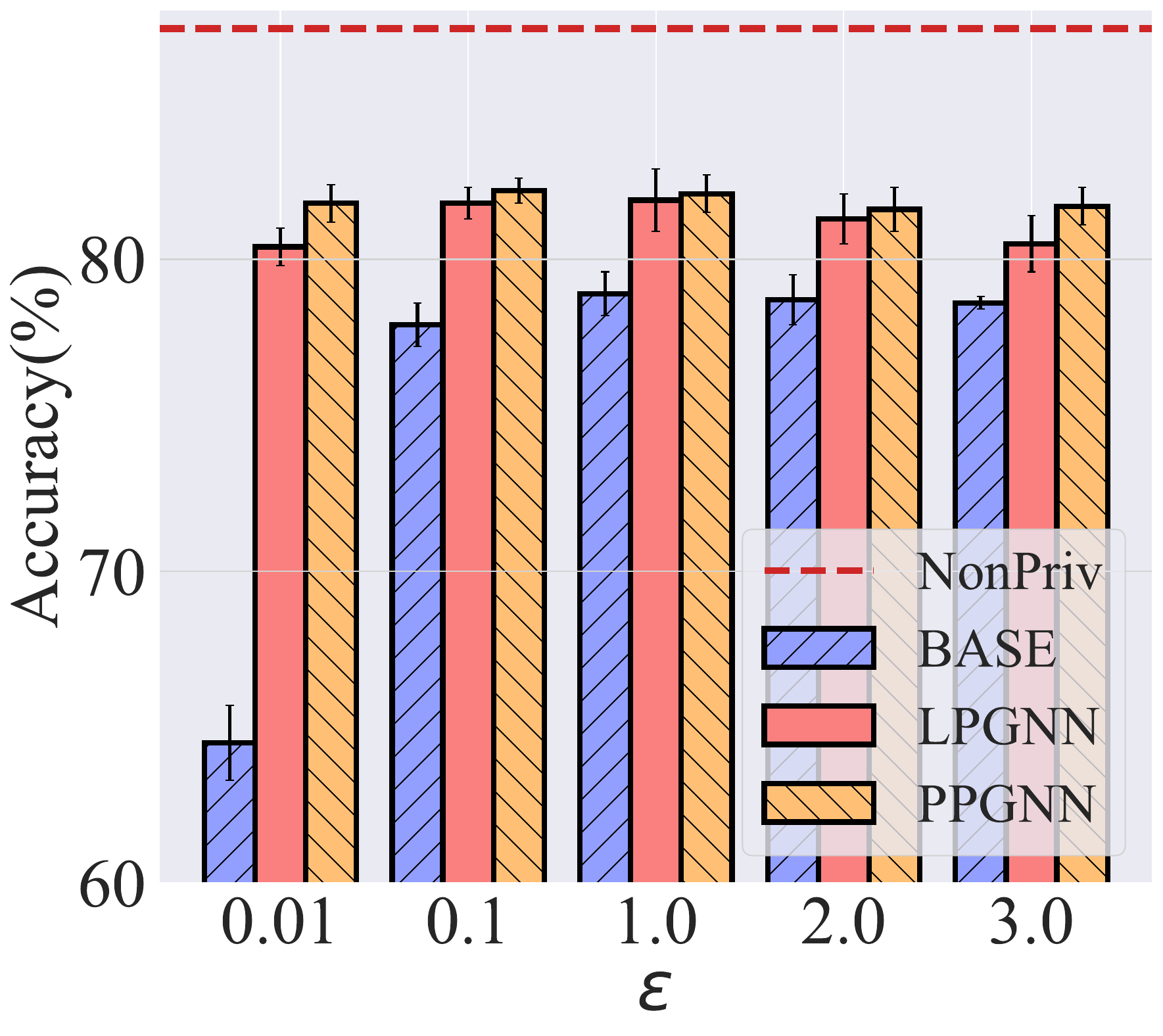} & 
	\includegraphics[width=0.23\linewidth]{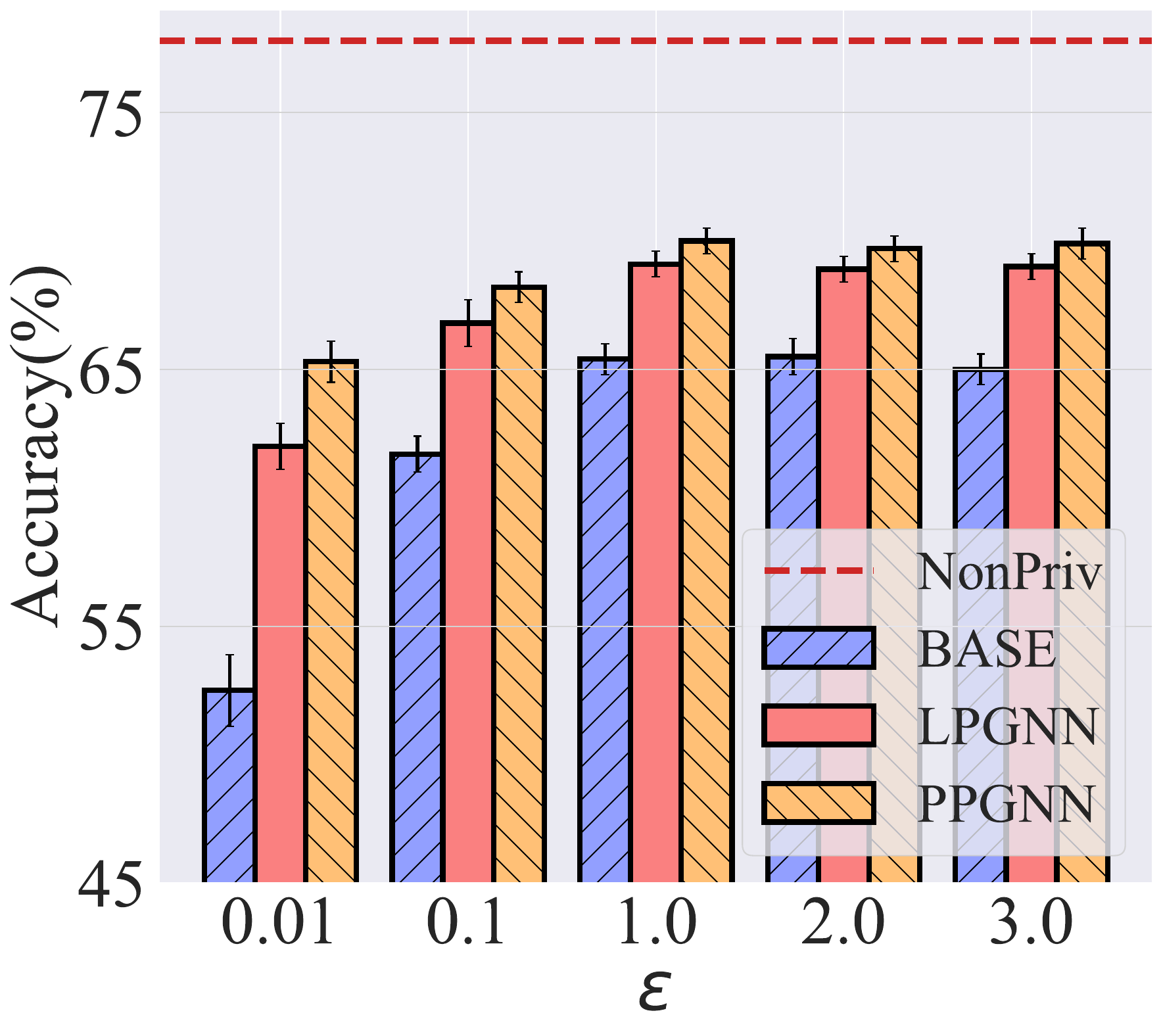}  \\
  \textbf{(e) Facebook} (GraphSAGE) & \textbf{(f) Wikipedia} (GraphSAGE) & \textbf{(g) Cora} (GCN) & \textbf{(h) CiteSeer} (GCN) \\
  \multicolumn{4}{c}{\textbf{}}\\ \includegraphics[width=0.23\linewidth]{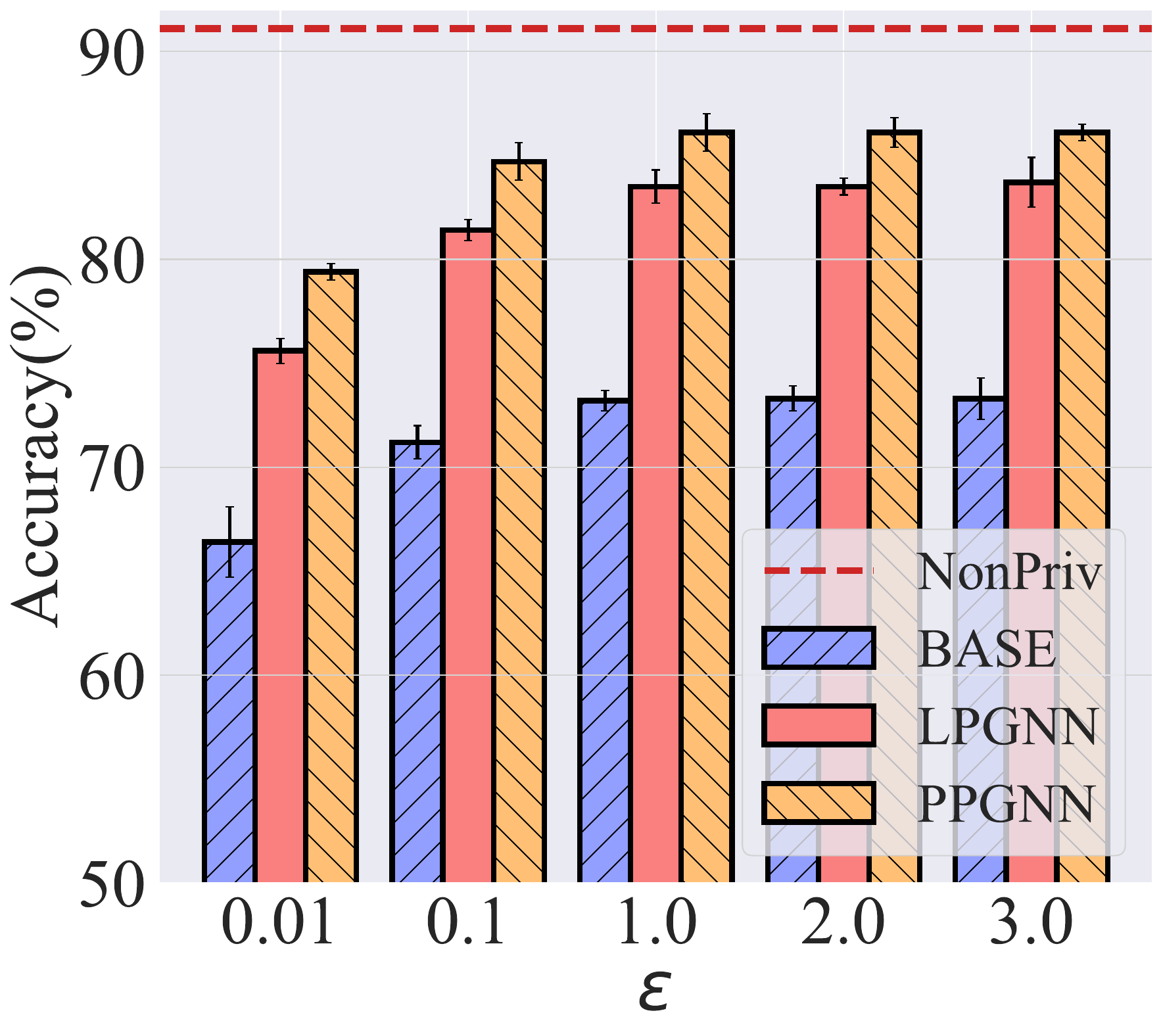} & 
	\includegraphics[width=0.23\linewidth]{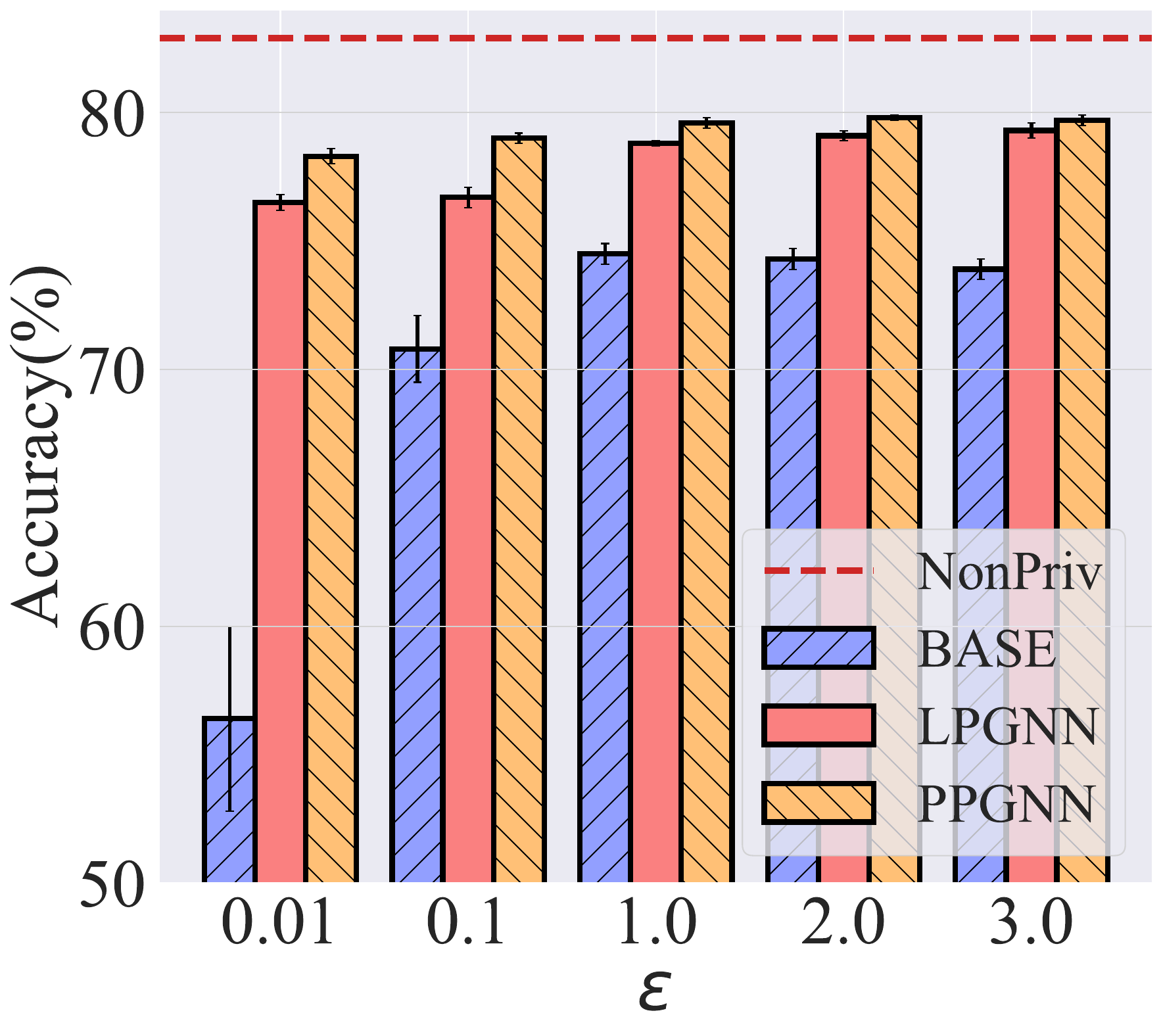} & 
	\includegraphics[width=0.23\linewidth]{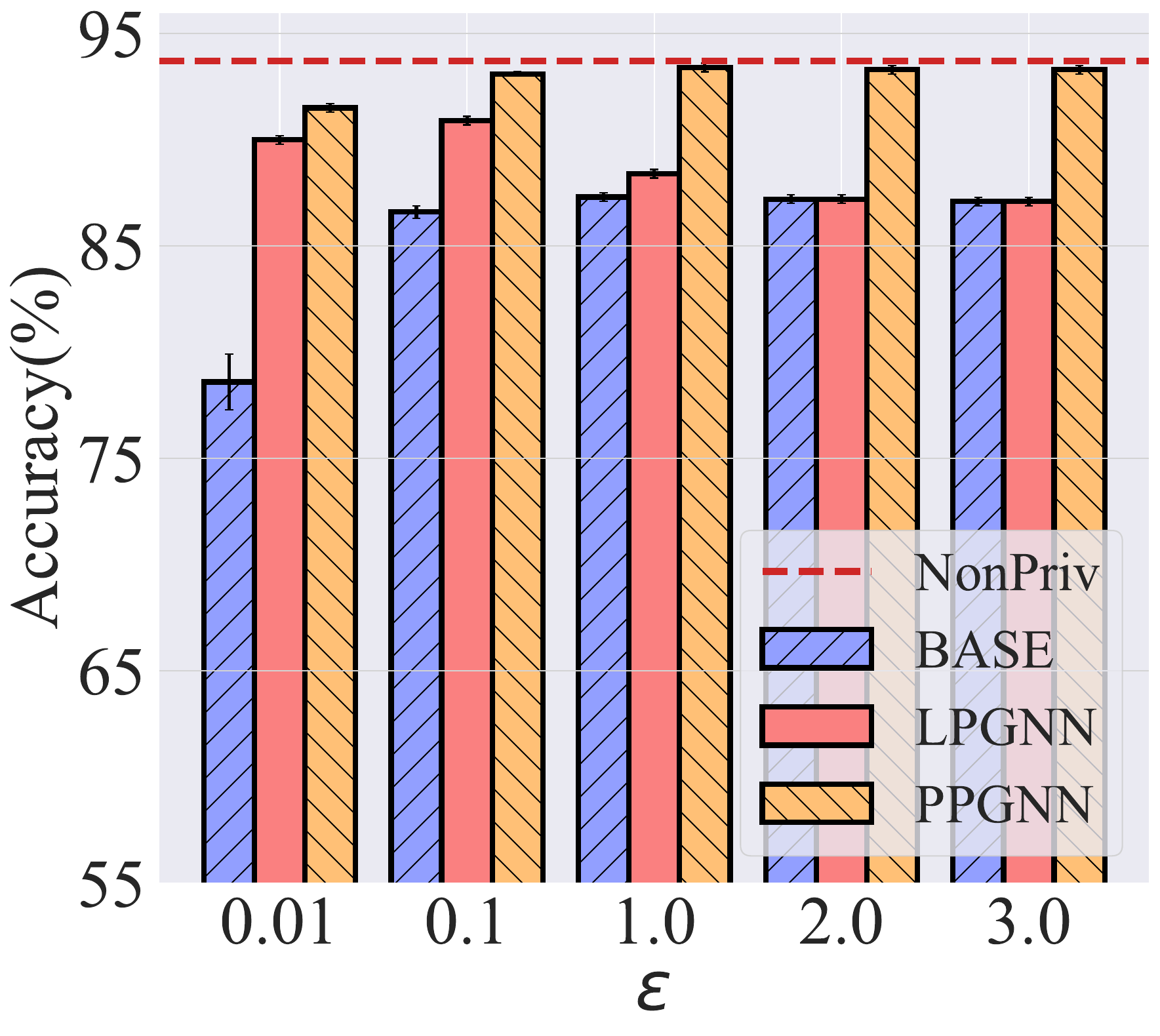} & 
	\includegraphics[width=0.23\linewidth]{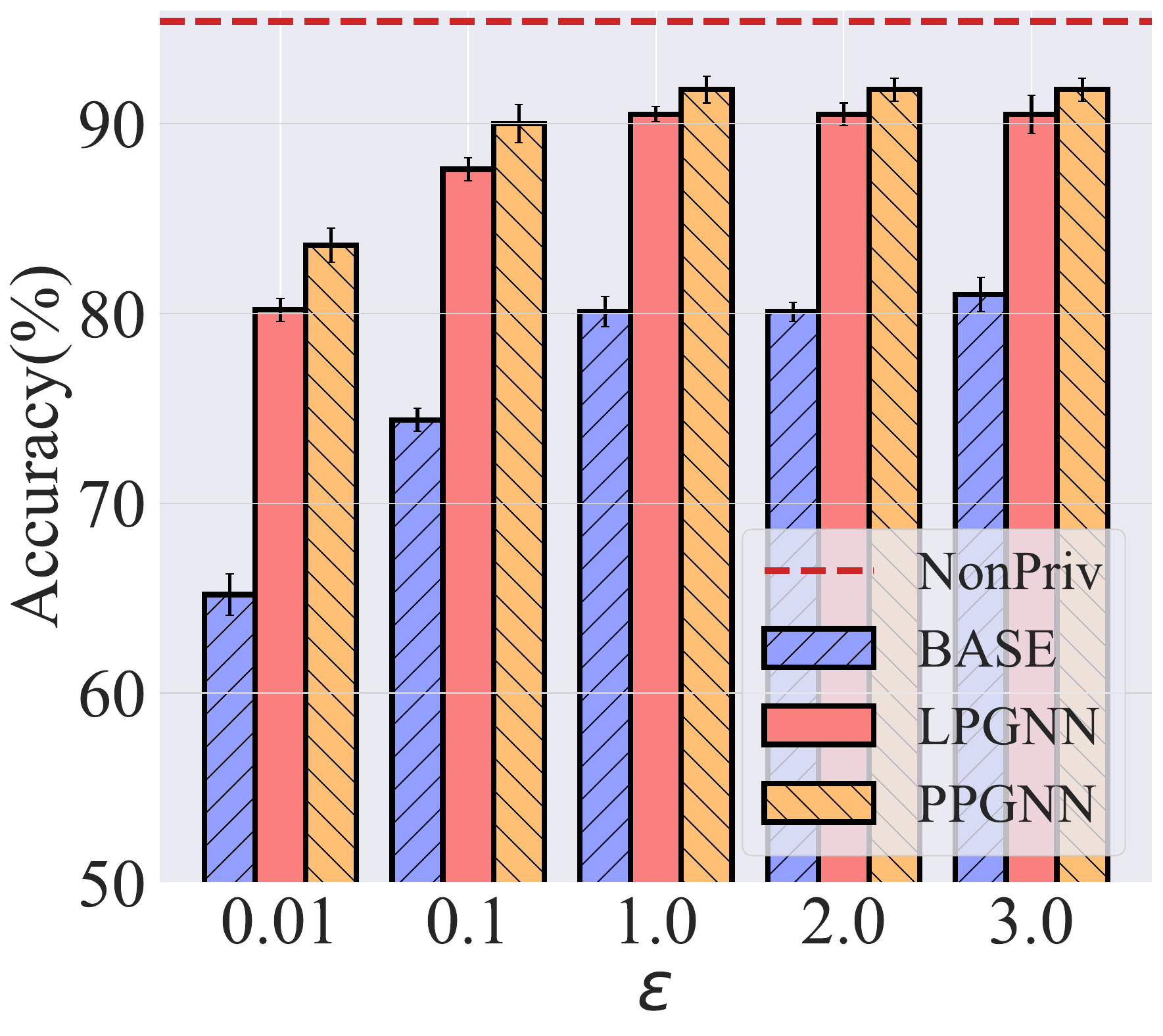}\\
 \textbf{(i) Pubmed} (GCN) & \textbf{(j) LastFM} (GCN) & \textbf{(k) Facebook} (GCN) & \textbf{(l) Wikipedia} (GCN)\\
	\end{tabular}
	\caption{Performance of PPGNN and other baselines. X-axis represents $\epsilon$ and y-axis represents test accuracy (\%).}
	\label{fig:3}
\end{figure*}

\subsubsection{Baselines} 
To demonstrate the effectiveness of PPGNN, we compare it with the following baselines: \textit{NonPriv}, \textit{BASE}, and \textit{LPGNN}~\cite{sajadmanesh2021locally}.\footnote{Considering that the node feature protection method in~\cite{lin2022towards} aligns with LPGNN~\cite{sajadmanesh2021locally}, and the research context of~\cite{pei2023privacy} focuses on subgraph-level protection, a direct comparison is not necessary. Therefore, these works are not included as baselines in this study.} The \textit{NonPriv} employs clean node features for graph learning. In contrast, \textit{BASE} uses features perturbed by the LDP mechanism directly in graph learning, bypassing a separate noise calibration step. \textit{LPGNN}, recognized as the current state-of-the-art in locally private graph learning, perturbs node features using a multi-bit mechanism and calibrates noise using the KProp algorithm. Additionally, to evaluate the performance of PPM, we compare it with the 1-bit mechanism (\textit{1B})~\cite{ding2017collecting}, Laplace mechanism (\textit{LP})~\cite{yang2020local}, Gaussian mechanism (\textit{GM})~\cite{balle2018improving}, and multi-bit mechanism (\textit{MB})~\cite{sajadmanesh2021locally}.
\begin{table*}[t]
	\centering
	\small
	\sc
 \caption{Accuracy of different feature LDP mechanisms}
	\setlength{\tabcolsep}{4.2mm}\begin{tabular}{c|l|cccccc}
		\toprule
		$\epsilon$  & Mech. & Cora & CiteSeer & Pubmed &  LastFM & Facebook & Wikipedia       \\
		\midrule
		\multirow{5}*{$\epsilon=0.01$}  
        & 1b      & 25.9 $\pm$ 2.8  & 20.6 $\pm$ 1.2 &  52.5 $\pm$ 1.2  & 12.2 $\pm$ 3.3  & 30.3 $\pm$ 1.2  &    42.2 $\pm$ 2.5     \\
		& lp   & 23.3 $\pm$ 3.2   & 20.6 $\pm$ 1.3  & 58.1 $\pm$ 0.9 & 9.0 $\pm$ 1.8  &       31.5 $\pm$ 1.3   & 43.5 $\pm$ 2.1  \\
        & gm      &  31.3 $\pm$ 2.6 & 25.4 $\pm$ 1.8 &  59.2 $\pm$ 1.5  & 28.2 $\pm$ 2.1 &  44.7 $\pm$ 1.3  &    51.4 $\pm$ 2.5    \\
        & mb   & 55.1 $\pm$ 2.2   & 43.1 $\pm$ 1.7  & 66.2 $\pm$ 0.7 & 62.5 $\pm$ 1.9  &      76.7 $\pm$ 1.3 & 63.6 $\pm$ 2.3    \\
        & ppm  & \textbf{62.0 $\pm$ 2.1}  & \textbf{47.3 $\pm$ 1.8}   & \textbf{68.6 $\pm$ 1.1} & \textbf{66.4 $\pm$ 1.6}  &  \textbf{79.5 $\pm$ 1.4}    & \textbf{67.2 $\pm$ 2.6}     \\	
		\midrule
		\multirow{5}*{$\epsilon=0.1$}  
        & 1b      &  30.2 $\pm$ 1.2  & 22.5 $\pm$ 1.0 &  57.5 $\pm$ 2.3  & 18.8 $\pm$ 2.8  & 34.5 $\pm$ 0.9 &     47.1 $\pm$ 2.2     \\
		& lp       &  32.5 $\pm$ 0.9 & 21.0 $\pm$ 1.5 &  56.1 $\pm$ 2.1  & 18.3 $\pm$ 1.6  &  35.1 $\pm$ 1.3  &   49.2 $\pm$ 2.2       \\
        & gm      & 37.7 $\pm$ 1.2  &  27.8 $\pm$ 0.9 & 62.2 $\pm$ 1.8   & 41.2 $\pm$ 1.1  &  61.4 $\pm$ 0.8 &    62.6 $\pm$ 2.1     \\
        & mb      &  72.3 $\pm$ 1.6 & 57.7 $\pm$ 0.8 &  71.6 $\pm$ 1.9   &  74.2$\pm$ 1.5 &  86.5 $\pm$ 0.4 &   72.3 $\pm$ 1.8       \\
        & ppm    & \textbf{76.1 $\pm$ 0.8}  & \textbf{61.0 $\pm$ 0.8} &  \textbf{75.3 $\pm$ 2.2}  & \textbf{78.1 $\pm$ 0.9}  & \textbf{88.7 $\pm$ 0.3} &  \textbf{73.9 $\pm$ 2.5}      \\
      \midrule
		\multirow{5}*{$\epsilon=1.0$}  
        & 1b      & 42.8 $\pm$ 1.6  & 27.1 $\pm$ 1.4 & 62.4 $\pm$ 3.3   & 31.0 $\pm$ 3.6  &  63.2 $\pm$ 1.5 &  55.6 $\pm$ 1.7         \\
		& lp       &  39.7 $\pm$ 3.0 & 25.8 $\pm$ 2.1 &  62.7 $\pm$ 2.9  &  29.4 $\pm$ 3.6  &58.7 $\pm$ 2.2  & 58.4 $\pm$ 2.2         \\
        & gm      &  46.5 $\pm$ 2.1  &  31.2 $\pm$ 1.8  &  65.1 $\pm$ 1.4  &   49.1 $\pm$ 2.5 &  71.2 $\pm$ 1.4 &   70.2 $\pm$ 1.9      \\
        & mb      & 78.9 $\pm$ 1.2   & 63.2 $\pm$ 0.5 &  72.1 $\pm$ 2.4  & 81.0 $\pm$ 0.3  & 90.7 $\pm$ 0.1  &  74.1 $\pm$ 1.8        \\
        & ppm    & \textbf{79.1 $\pm$ 0.5}  & \textbf{63.3 $\pm$ 0.5} & \textbf{75.4 $\pm$ 2.1}   &   \textbf{81.2 $\pm$ 0.6} & \textbf{90.9 $\pm$ 0.2} &   \textbf{79.8 $\pm$ 1.9}       \\
		\bottomrule
	\end{tabular}
  \label{tab2}
\end{table*}
        
\subsubsection{Experiment setup} 
For each dataset, we randomly divide 50\% of the data samples as the training set, 25\% as the validation set and 25\% as the test set. {We select three GNN architectures as backbone models: \textbf{GCN}~\cite{DBLP:conf/iclr/KipfW17}, which aggregates neighborhood information via a symmetric normalized adjacency matrix; \textbf{GraphSAGE}~\cite{hamilton2017inductive}, which samples a fixed-size neighborhood and aggregates features using learnable aggregator functions; and \textbf{GAT}~\cite{velivckovic2018graph}, which introduces an attention mechanism to assign learned importance weights to different neighbors during aggregation.} All GNN models have two graph convolutional layers, each with a hidden dimension of size 16 and a SeLU activation function followed by dropout. We implement the GNN models in PyTorch using PyTorch-Geometric (PyG)\footnote{https://www.pyg.org}. All experiments are carried out on a machine running Ubuntu 20.04 LTS, equipped with two Intel\textsuperscript{\textregistered} Xeon\textsuperscript{\textregistered} Gold 6348 CPUs, 100GB RAM, and an NVIDIA\textsuperscript{\textregistered} A800 80GB GPU. We set $\mathcal{E}=\{\epsilon, 2\epsilon,\dots,2^h\epsilon\}$, where $\epsilon \in \{0.01, 1, 2, 3\}$. If not specified otherwise, $h$ is set to 5. To obtain the best hyperparameters, we
use grid search for selection: both learning rate and weight
decay are chosen from $\{10^{-4},10^{-3},10^{-2},10^{-1},0\}$, the dropout rate is chosen from $\{0.75,0.5,0.25,0\}$, and the allocation parameter $\delta$ is chosen from $\{0.1, 0.3, 0.5, 0.7, 0.9\}$. The optimal value of the FlexProp's step parameter $K$ for different $\epsilon \in \{0.01, 1, 2, 3\}$ is selected from the set $\{0, 2, 4, 8, 16\}$. We used the Adam optimizer~\cite{kingma2014adam} for all models. We utilize a random function to generate users’ privacy levels.

\subsubsection{Evaluation metrics}
We use node classification accuracy as a metric to assess the performance of PPGNN with different parameters. All models are trained for 500 epochs, and the best model is selected for testing based on validation loss. We measure the accuracy of the test set over 10 consecutive runs and report the average and 95\% confidence intervals calculated by bootstrapping over 1000 samples.

\subsection{Main Experimental Results}
Fig.~\ref{fig:3} demonstrates the accuracy of PPGNN with various baselines for all cases. In this experiment, the BASE method employs the MB~\cite{sajadmanesh2021locally} to perturb node features. As depicted in Fig.~\ref{fig:3}, PPGNN consistently outperforms both BASE and LPGNN, significantly narrowing the gap with NonPriv. Notably, in some cases, the accuracy of PPGNN closely approaches that of NonPriv. Specifically, as shown in Fig.~\ref{fig:3}(a), when $\epsilon=0.01$, LPGNN and PPGNN exhibit accuracy 24.5\% and 26.8\% higher than that of BASE, respectively. This not only underscores the effectiveness of the noise calibration strategies employed by LPGNN and PPGNN but also emphasizes that relying solely on the LDP mechanism, without adequate noise calibration, leads to diminished utility. Furthermore, as evidenced in Fig.~\ref{fig:3}(e) and (k), PPGNN's accuracy aligns closely with that of NonPriv. Collectively, these observations validate the robust efficacy of PPGNN.

\begin{figure}
	\centering
	\begin{tabular}{cc}
		\multicolumn{2}{c}{\textbf{}} \\ \vspace{-0.5em}  \includegraphics[width=0.45\linewidth]{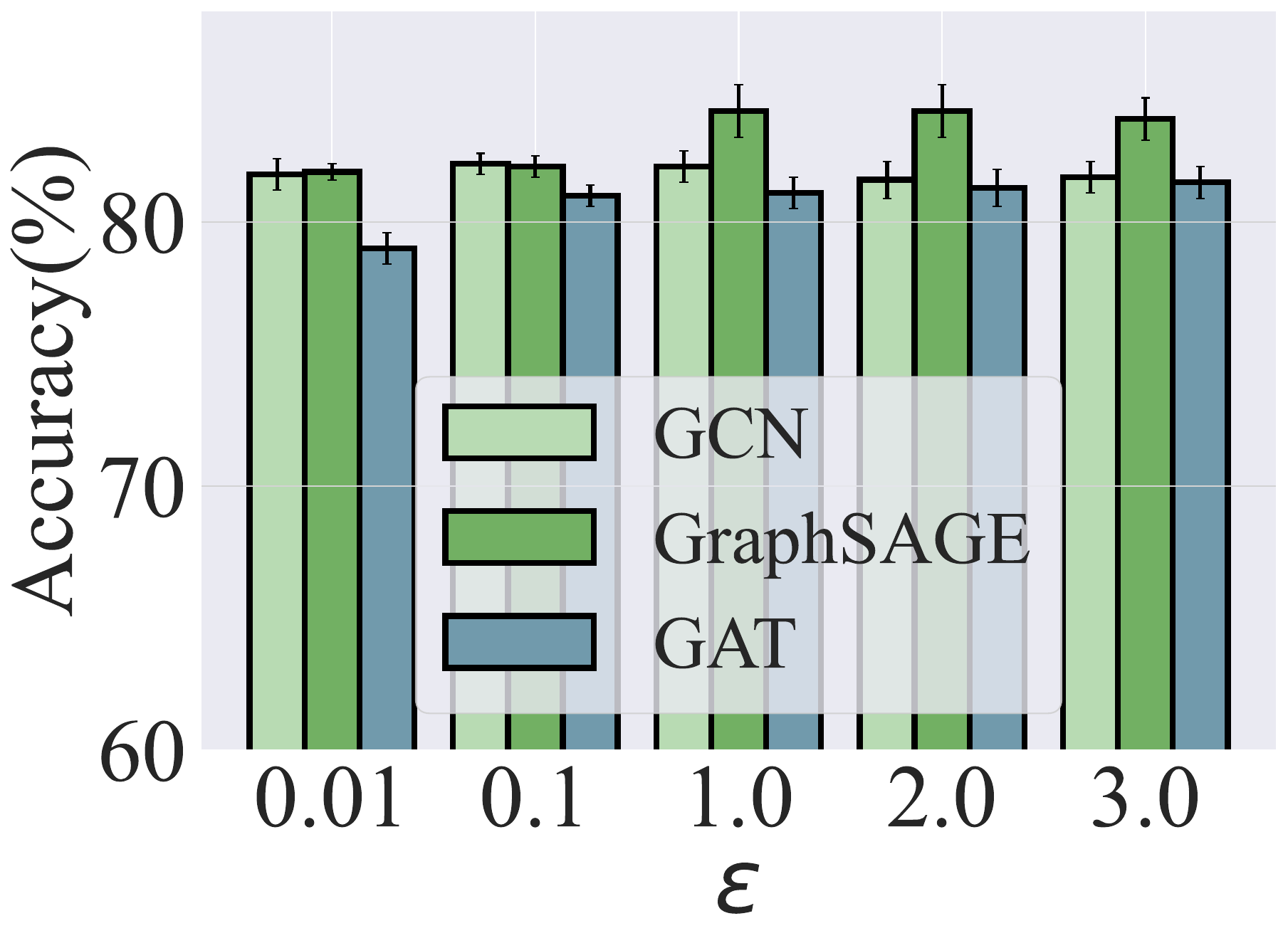} & 
	\includegraphics[width=0.45\linewidth]{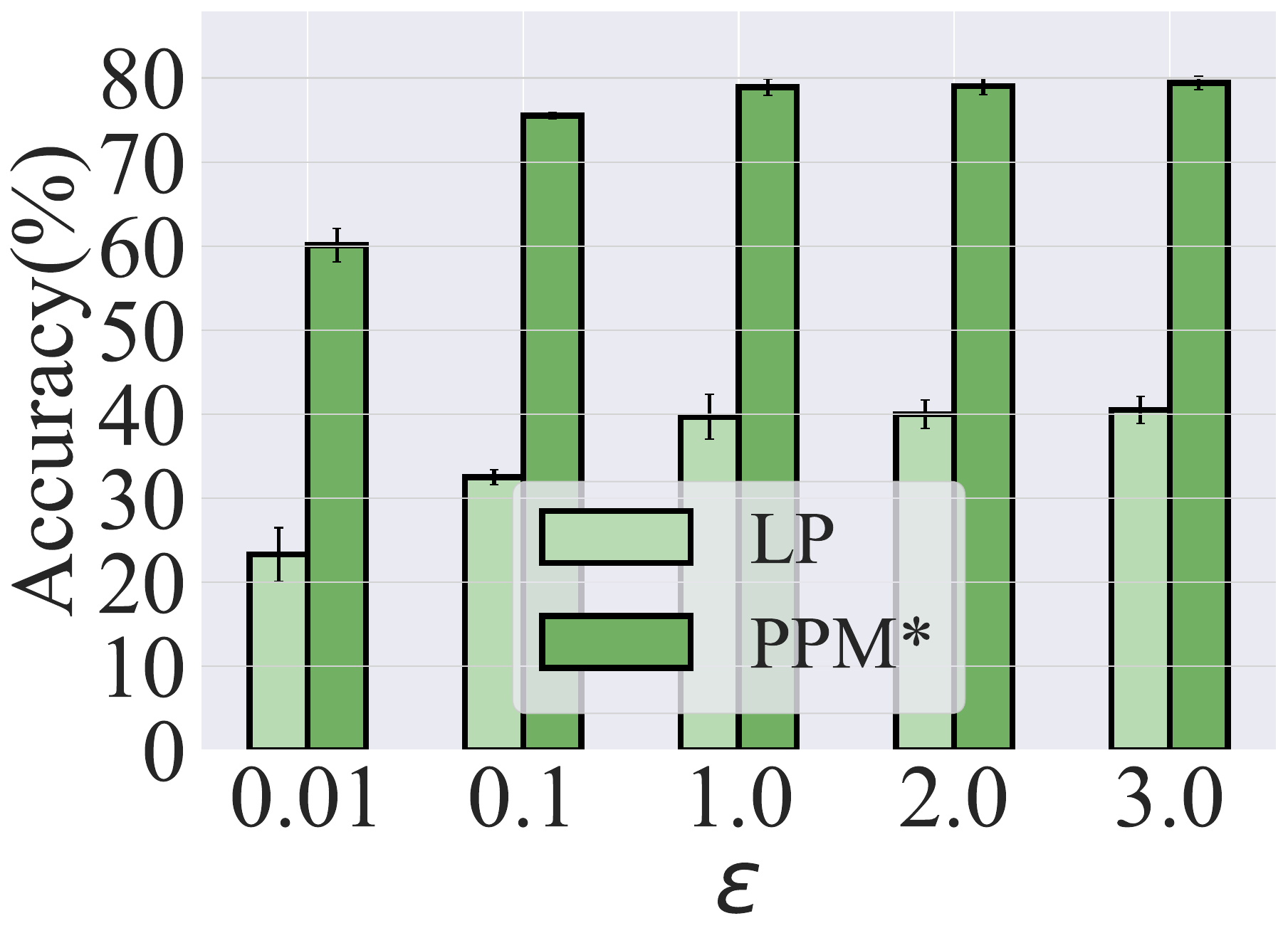} 
	\end{tabular}
    \vspace{-0.5em}
	 \caption{\textbf{Left:} Comparison of the performance of PPGNN under different GNN models. \textbf{Right:} Performance of PPM* (LP integrated) vs. vanilla LP under different $\epsilon$ values, demonstrating utility gains from integrating LP into MLR.}
     \vspace{-1em}
	\label{fig:new2}
\end{figure}
To further understand the impact of different GNN architectures under personalized LDP settings, we compare the performance of PPGNN when instantiated with GCN~\cite{DBLP:conf/iclr/KipfW17}, GraphSAGE~\cite{hamilton2017inductive}, and GAT~\cite{velivckovic2018graph} backbones. As shown in the left panel of Fig.~\ref{fig:new2}, all three models benefit from the proposed framework, but their sensitivity to noise varies. Notably, GAT tends to underperform GCN and GraphSAGE when the $\epsilon$ is small (\textit{e.g.}, $0.01$). We attribute this to GAT’s attention mechanism, which dynamically learns neighbor importance scores based on feature similarity. When input features are heavily perturbed, this attention computation becomes unstable, resulting in greater performance degradation. In contrast, GCN and GraphSAGE adopt more stable aggregation strategies, making them less sensitive to noise introduced by local perturbations. However, as $\epsilon$ increases (i.e., less noise), this gap gradually narrows, and the three models converge in performance.

\subsection{Robustness Under Extreme Privacy Distributions}

To further evaluate the adaptability of PPGNN in realistic settings, we simulate three extreme user privacy level distributions on the Cora dataset with GCN: \ding{172} \textit{Strict Privacy}: all users adopt the lowest privacy budget $\epsilon$; \ding{173} \textit{Relaxed Privacy}: all users adopt the highest budget $2^h\epsilon$; \ding{174} \textit{Bimodal}: half of the users adopt $\epsilon$, the other half $2^h\epsilon$.
As shown in the right panel of Fig.~\ref{fig:new1}, PPGNN consistently outperforms BASE in all cases. Under strict privacy (Case 1), PPGNN and LPGNN achieve similar performance, while both clearly surpass BASE. When privacy becomes more relaxed (Cases 2 and 3), PPGNN benefits more from its personalized design and achieves the best utility. These results confirm that PPGNN remains effective and robust under diverse privacy-level distributions.

\subsection{Ablation Study}
\subsubsection{Analyzing the Performance of PPM}
In Table~\ref{tab2}, we compare the performance of our Personalized Perturbation Mechanism (PPM) with the 1-bit mechanism (1B), the Laplace mechanism (LP), the Gaussian mechanism (GM), and the Multi-bit mechanism (MB). 1B~\cite{ding2017collecting}, GM~\cite{balle2018improving}, and LP~\cite{yang2020local} are widely used LDP mechanisms for locally perturbing single-value and multidimensional data. In this experiment, 1B, GM, and LP perturb each bit of the node features. MB is the LDP mechanism employed in LPGNN and currently stands as the most effective LDP algorithm for node features. According to the results in Table~\ref{tab2}, the accuracy of PPM surpasses that of other mechanisms across almost all cases, especially under smaller $\epsilon$. Specifically, when $\epsilon=0.01$ and the dataset is Cora, PPM achieves an accuracy that is approximately 7\%, higher accuracy on average compared to the second-best mechanism. This underscores the efficacy of PPM. To further demonstrate the effectiveness and generality of our MLR-based design, we conduct a more detailed analysis in Fig.~\ref{fig:new2} Right, where we compare PPM* (i.e., LP~\cite{yang2020local} integrated into the MLR) against the vanilla LP across different $\epsilon$. As shown in the figure, PPM* consistently outperforms LP, especially under tighter budgets. This verifies the utility gains and broad applicability of MLR when combined with various perturbation strategies.

\subsubsection{Analyzing the Performance of FlexProp}

In Fig.~\ref{fig:66} and Fig.~\ref{fig:77}, we contrast our FlexProp with another noise calibration algorithm for node features, KProp, which is the noise calibration algorithm used in LPGNN~\cite{sajadmanesh2021locally}. Based on the results from Fig.~\ref{fig:66} and Fig.~\ref{fig:77}, the accuracy of FlexProp surpasses that of KProp across all cases, underscoring the effectiveness of FlexProp. Additionally, we observe that FlexProp exhibits a more significant improvement with a larger privacy budget $\epsilon$. This is attributed to the reduced noise injection due to the higher $\epsilon$, allowing FlexProp to demonstrate its advantages more prominently, particularly in scenarios with higher node feature variability and more complex graph structures. 
\begin{figure}
	\centering
	\begin{tabular}{cc}
		\multicolumn{2}{c}{\textbf{}} \\   \includegraphics[width=0.45\linewidth]{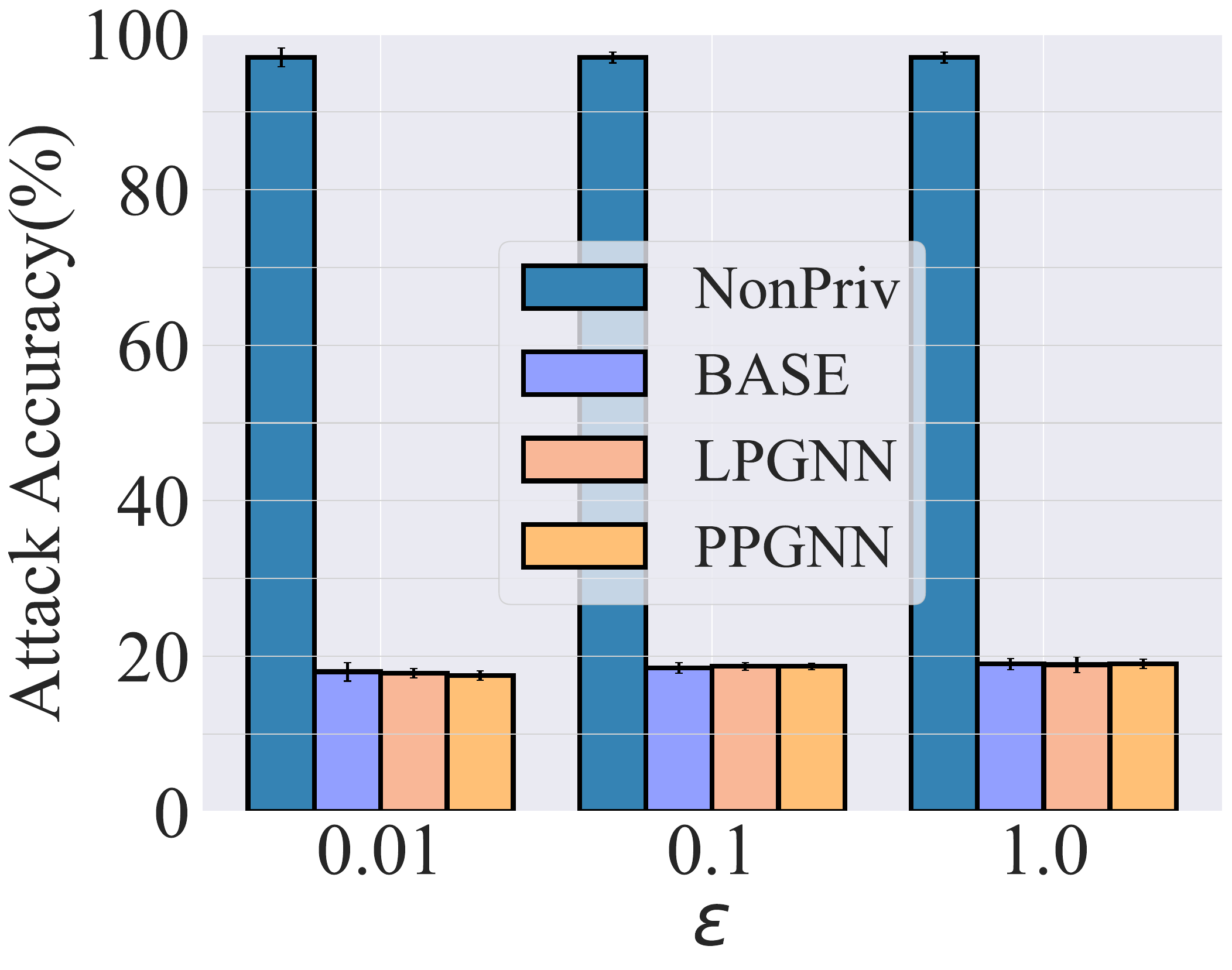} & 
\includegraphics[width=0.45\linewidth]{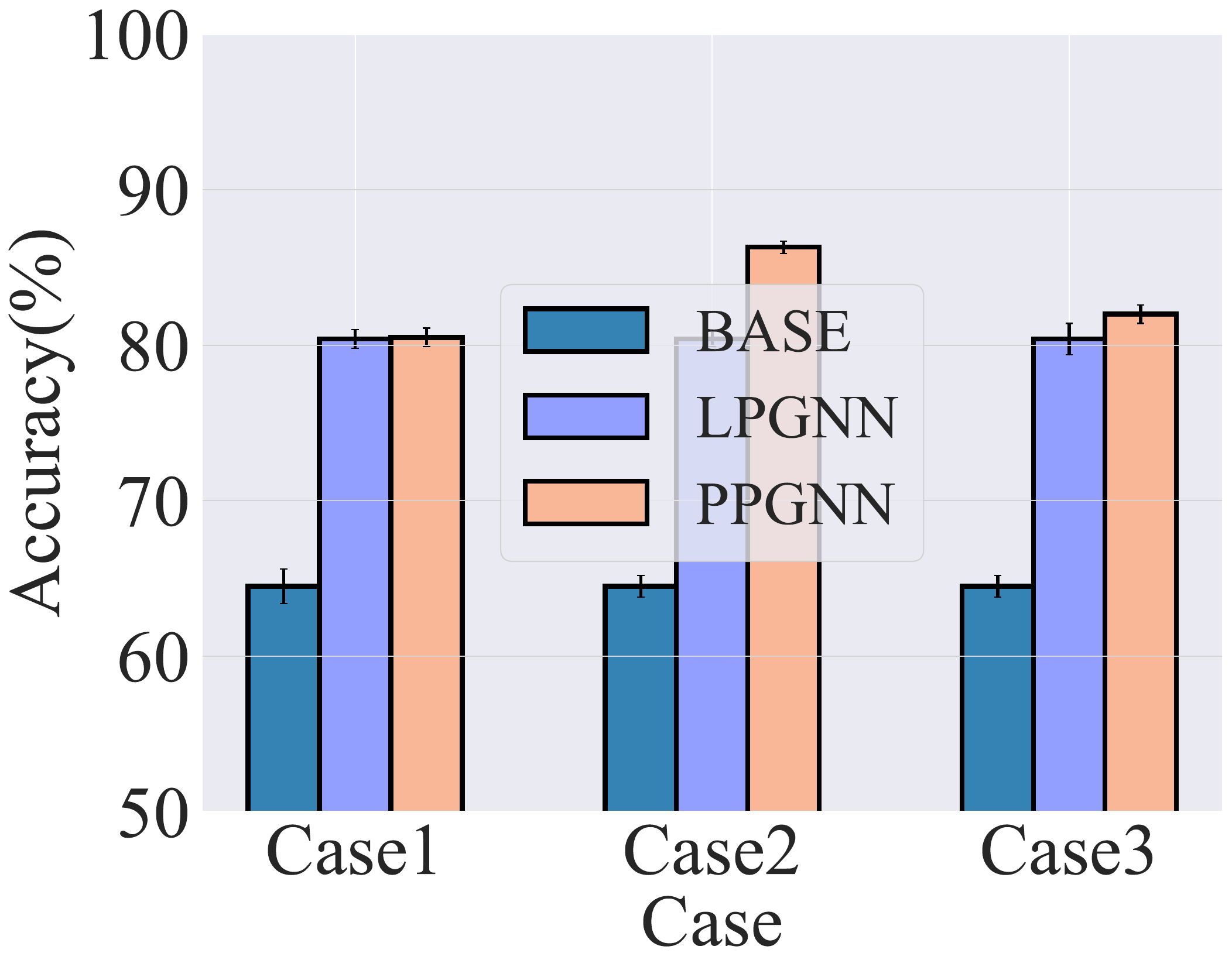} 
	\end{tabular}
    \vspace{-1em}
	 \caption{\textbf{Left:} Attack accuracy of attribute inference under different privacy budgets ($\epsilon$) on the Cora dataset. LDP-based methods (BASE, LPGNN, and PPGNN) effectively mitigate privacy leakage compared to NonPriv, while PPGNN preserves strong utility without compromising privacy. \textbf{Right:} Accuracy of PPGNN under three extreme privacy distributions on the Cora dataset using GCN. PPGNN demonstrates strong adaptability across different user privacy configurations, validating its personalized design.}
     \vspace{-1em}
	\label{fig:new1}
\end{figure}

\begin{figure}
	\centering
	\begin{tabular}{cc}
		\multicolumn{2}{c}{\textbf{}} \\   \includegraphics[width=0.45\linewidth]{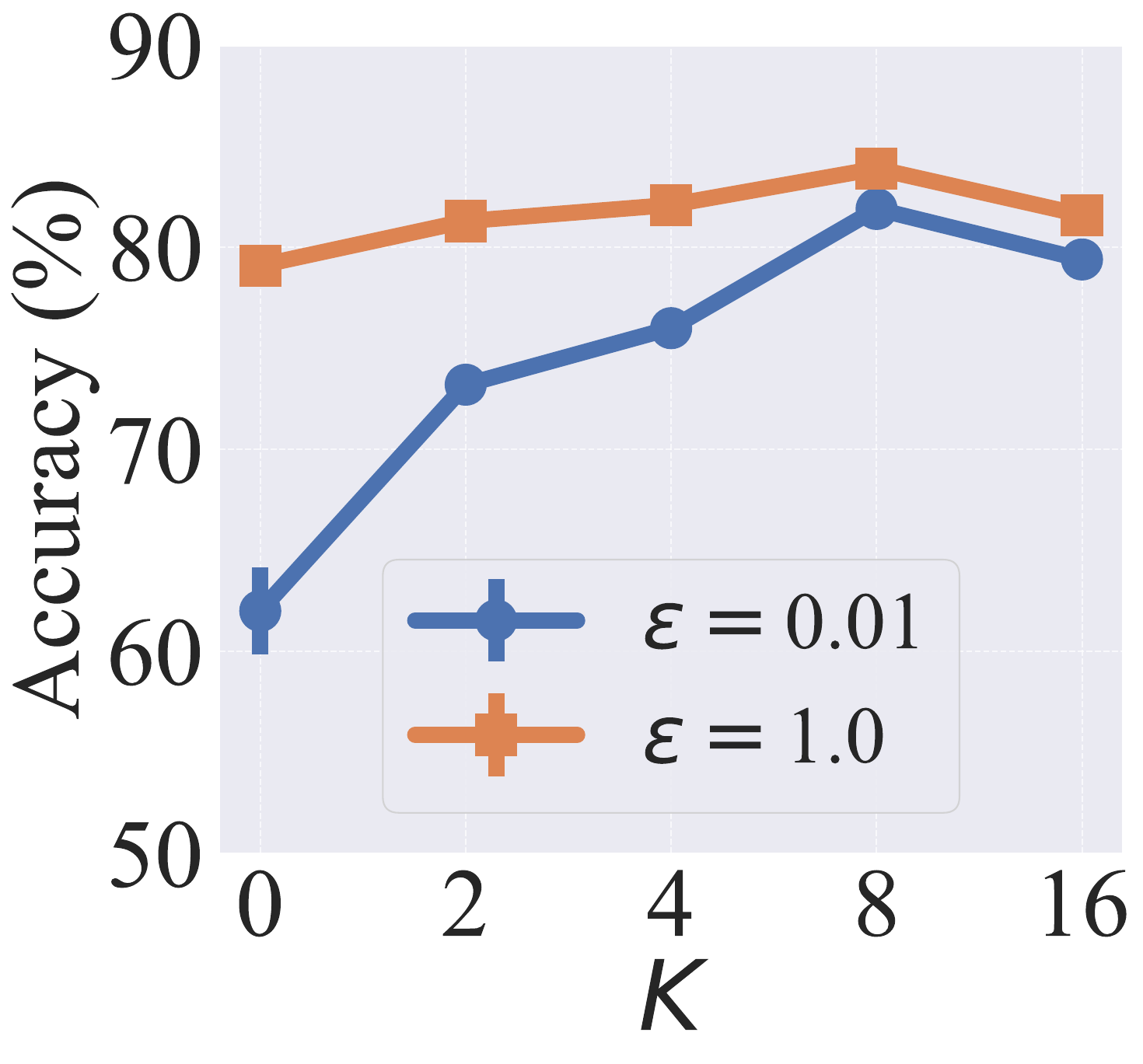} & 
	\includegraphics[width=0.45\linewidth]{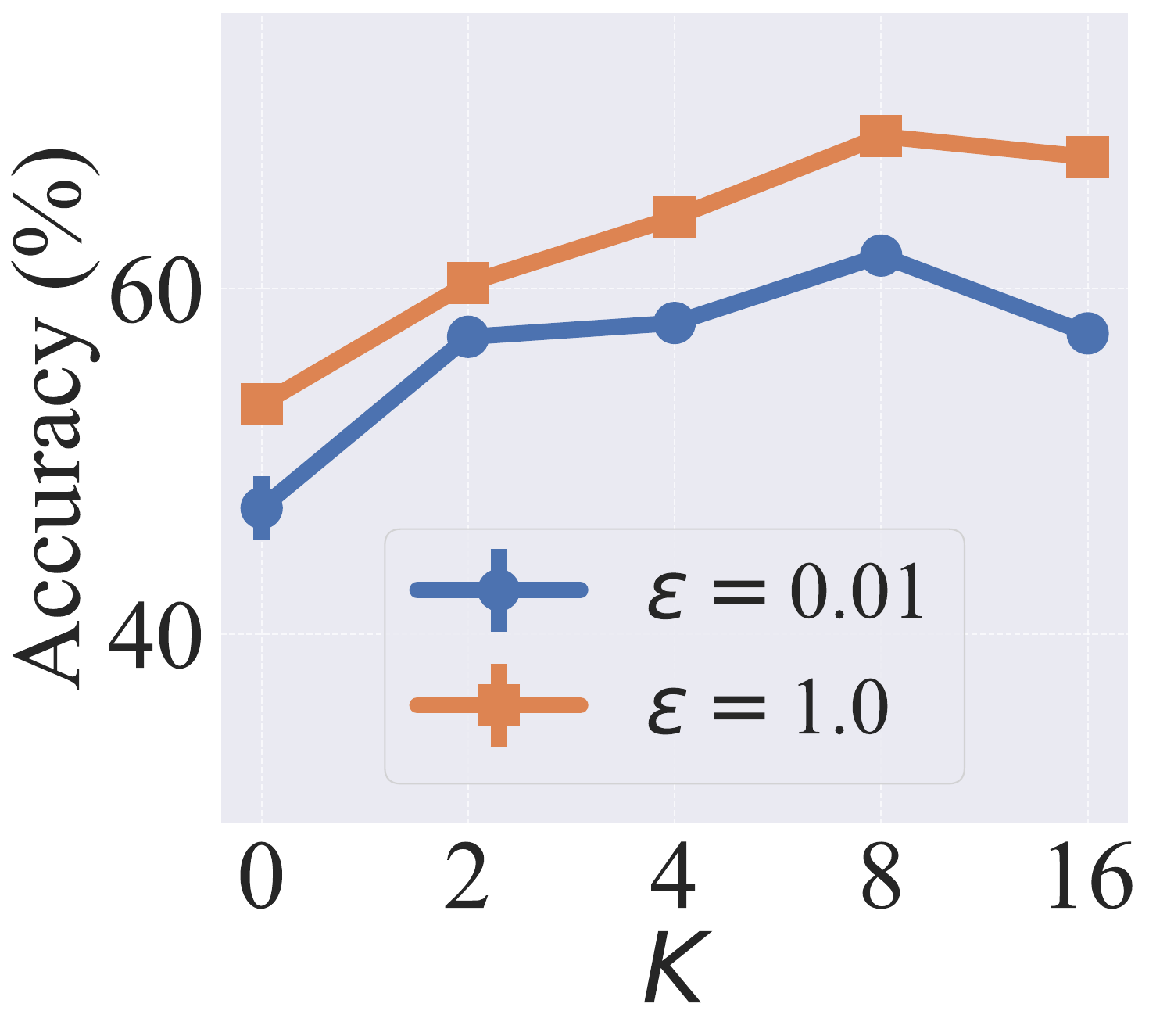} 
    \\(a) Cora & (b) CiteSeer  \\
	\end{tabular}
	 \caption{Accuracy of PPGNN across varied parameter $K$.}
	\label{fig:kkk}
\end{figure}

\subsection{Empirical Privacy Attack Defense}
To empirically assess the privacy protection of PPGNN, we conduct attribute inference attacks (AIA)~\cite{chenunderstanding,meng2023devil}. In this setting, an attacker observes the perturbed features of a target node’s neighbors and infers the target’s sensitive attribute via majority voting, leveraging homophily in the graph. We evaluate four methods, NonPriv, BASE, LPGNN, and PPGNN, on the Cora dataset under varying privacy budgets $\epsilon \in \{0.01, 0.1, 1.0\}$. As shown in the left panel of Fig.~\ref{fig:new1}, NonPriv suffers from high attack accuracy ($>90\%$), indicating severe privacy leakage. In contrast, all LDP-based methods achieve significantly lower attack accuracy, approaching random guessing, validating their strong defense against attribute inference. Notably, PPGNN exhibits similar protection to other LDP methods, as expected under the same privacy guarantees. Its key advantage lies in preserving comparable privacy while achieving superior utility, as demonstrated in earlier sections.

\subsection{Parameter Analysis}
\subsubsection{On the effects of $h$}
Fig.~\ref{fig:5}(a) displays the accuracy of PPGNN based on the GraphSAGE as the GNN backbone model and the Cora dataset, considering privacy budget $\epsilon\in\{0.01, 1.0\}$ and privacy level parameter $h\in\{1, 2, 3, 4, 5\}$. From Fig.~\ref{fig:5}(a), it is evident that as $h$ increases, the accuracy of PPGNN consistently improves for both $\epsilon=0.01$ and $\epsilon=1.0$. This trend occurs because as $h$ grows, users with lower privacy requirements are allocated a larger privacy budget, contributing to the enhancement of PPGNN's accuracy.

\begin{figure}
  \centering
  \vspace{-1.5em}
  \begin{tabular}{ccc}
		\multicolumn{3}{c}{\textbf{}} \\     \hspace{-1ex}\includegraphics[width=0.3\linewidth]{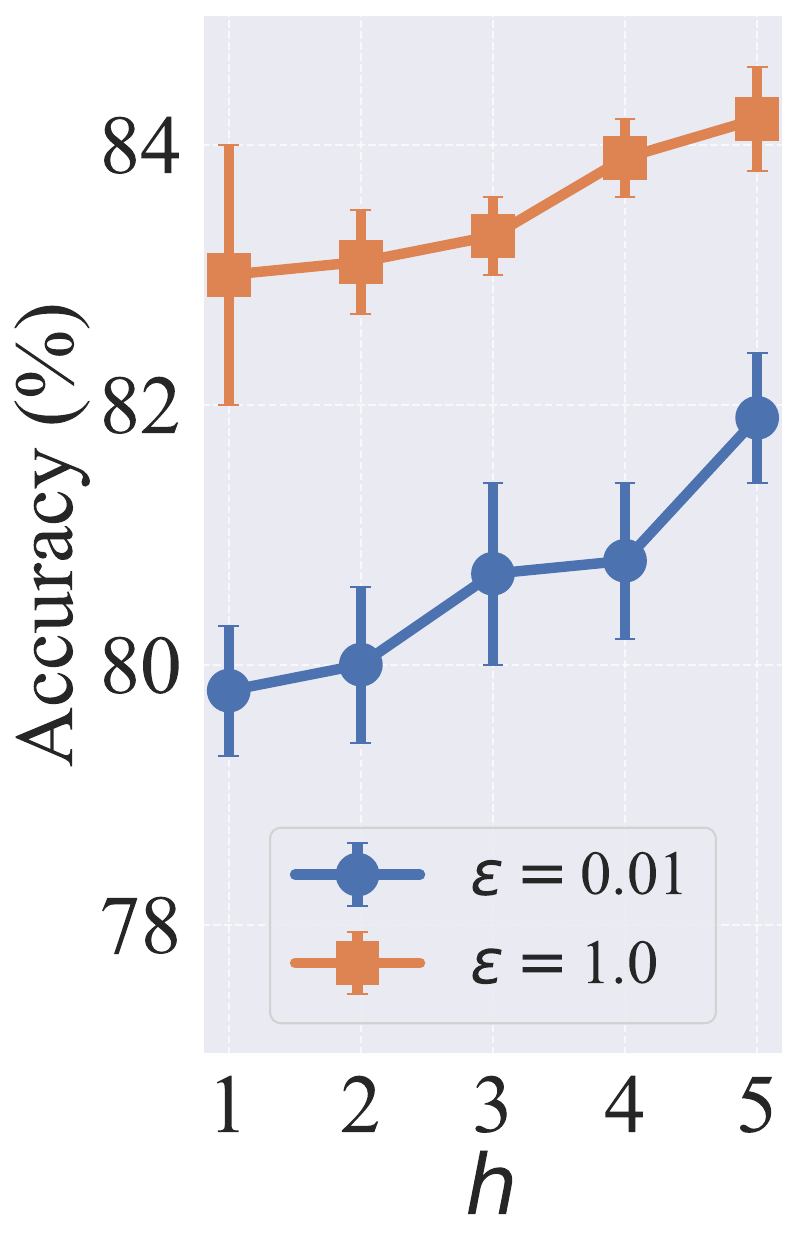} & 
	\includegraphics[width=0.3\linewidth]{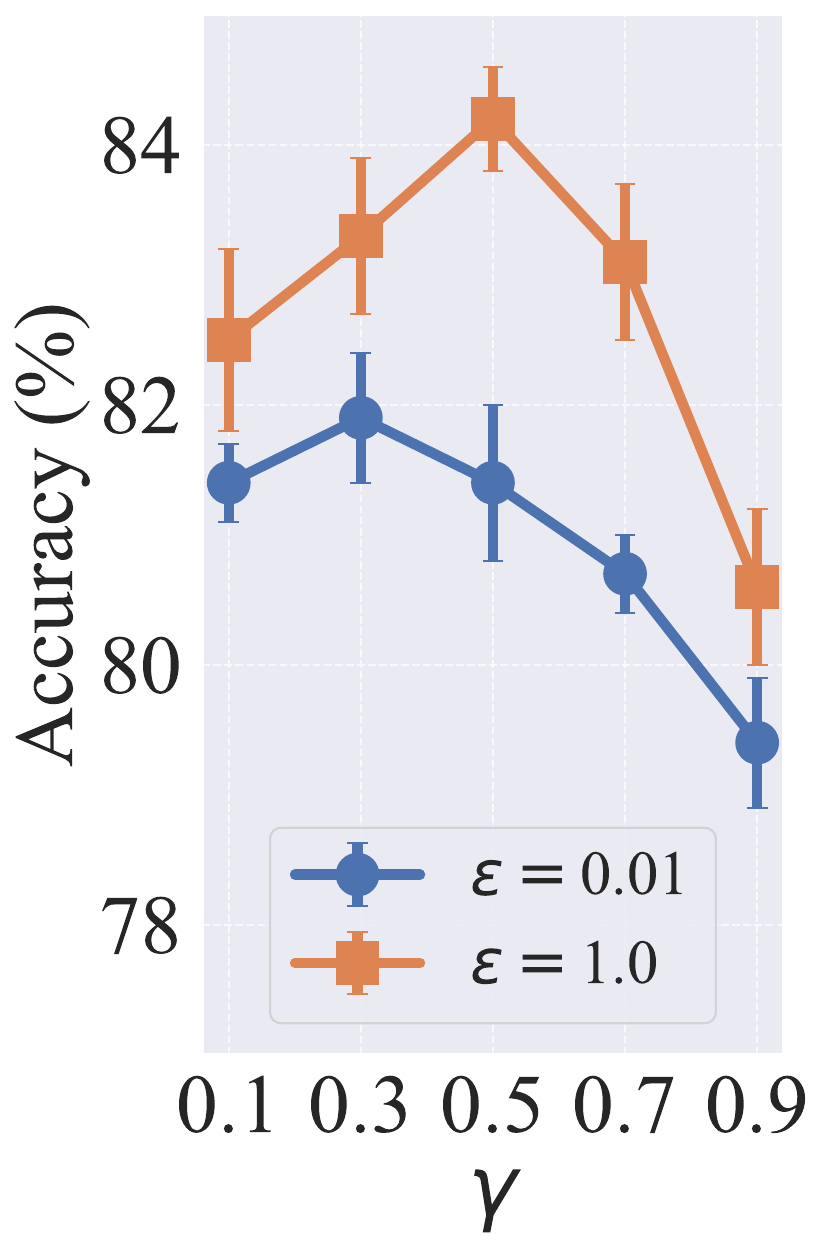}  & 
	\includegraphics[width=0.3\linewidth]{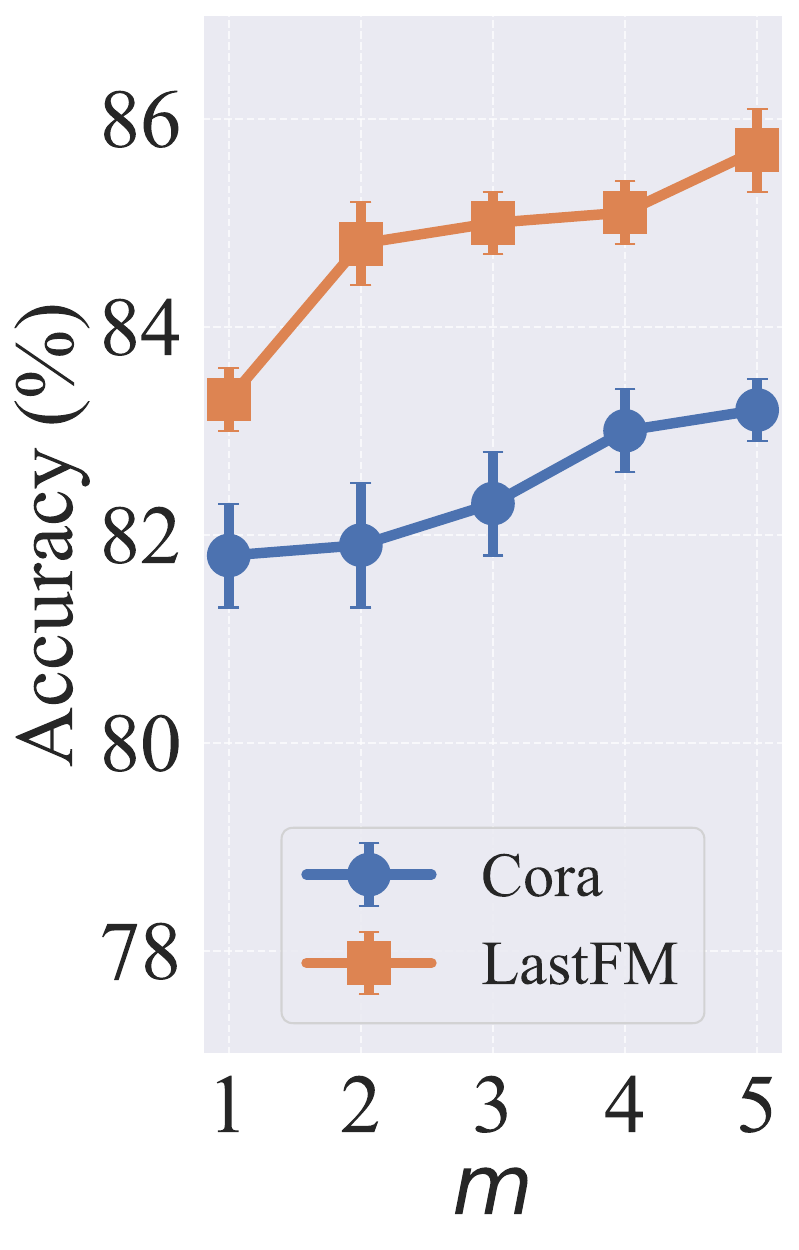} 
    \\(a) & (b) & (c) \\
	\end{tabular}
    
  \caption{(a): Accuracy of PPGNN across varied parameter $h$. (b): Accuracy of PPGNN across varied $\gamma$. (c): Accuracy of PPGNN across varied $m$.}
  \label{fig:5}
\end{figure}

\begin{figure*}[h!]
	\centering
	\begin{tabular}{cccccc}
		\multicolumn{6}{c}{\textbf{}} \\
   \vspace{-1em}\includegraphics[width=0.14\linewidth]{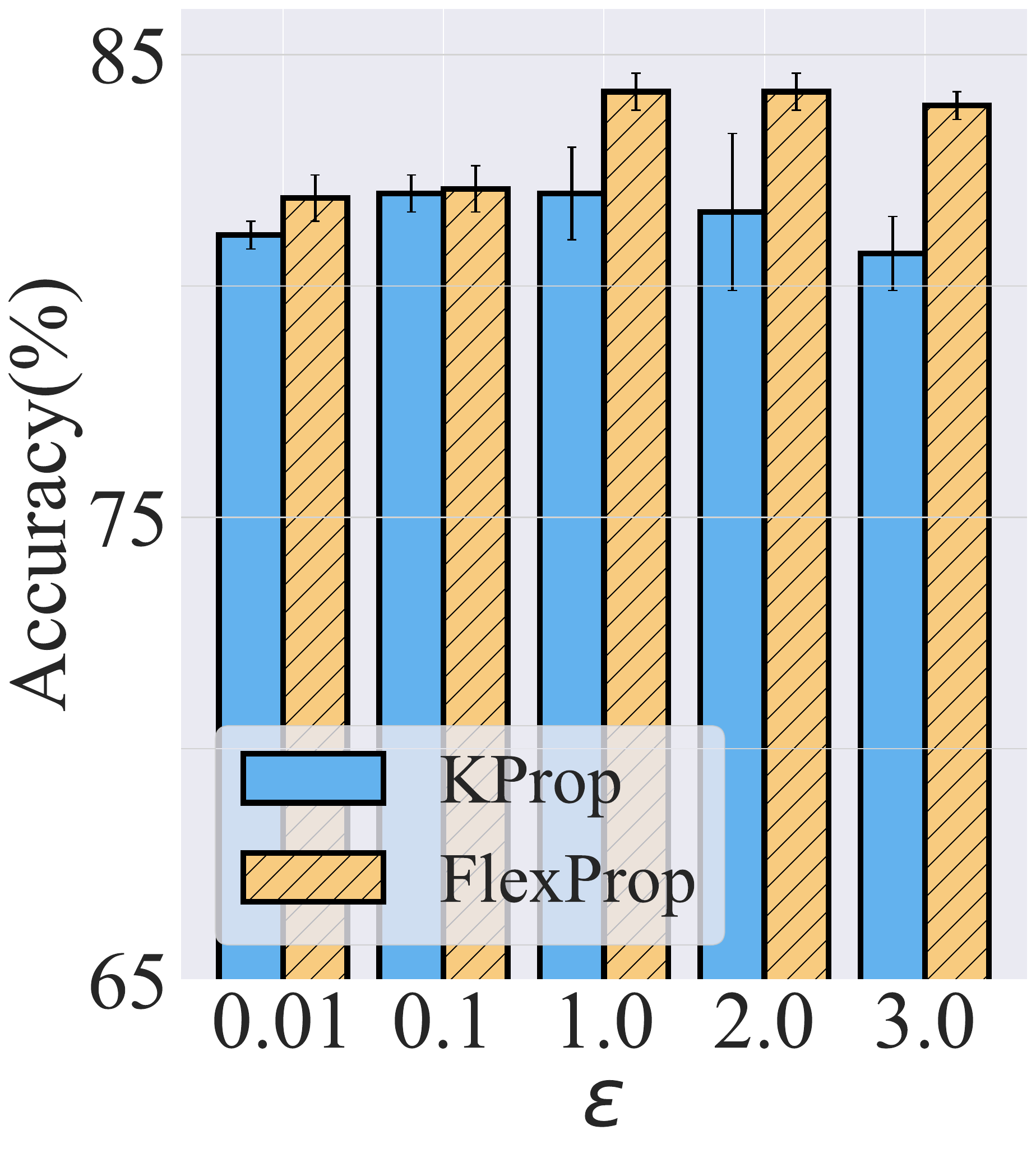} & 
	\includegraphics[width=0.14\linewidth]{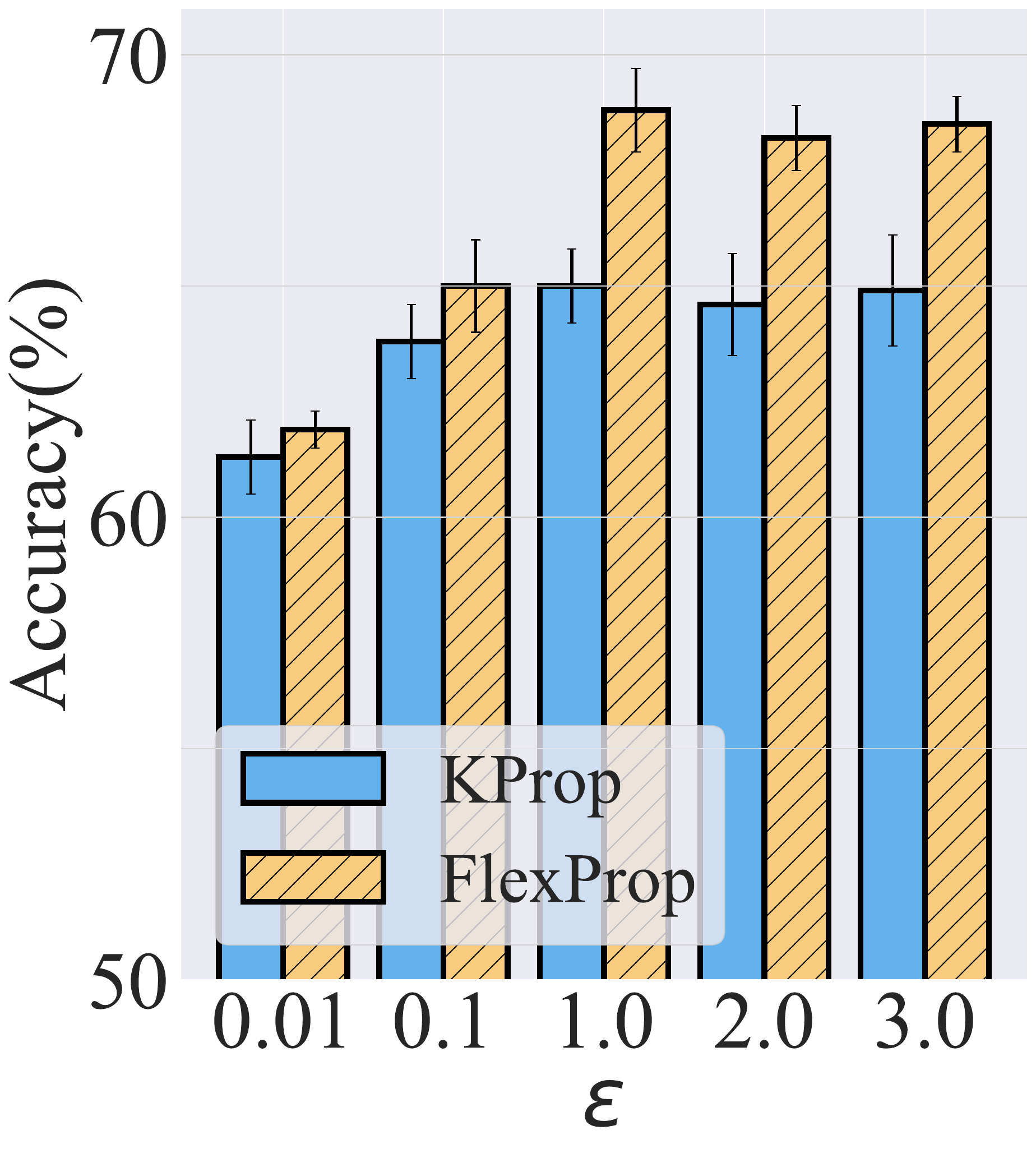} & 
    \includegraphics[width=0.14\linewidth]{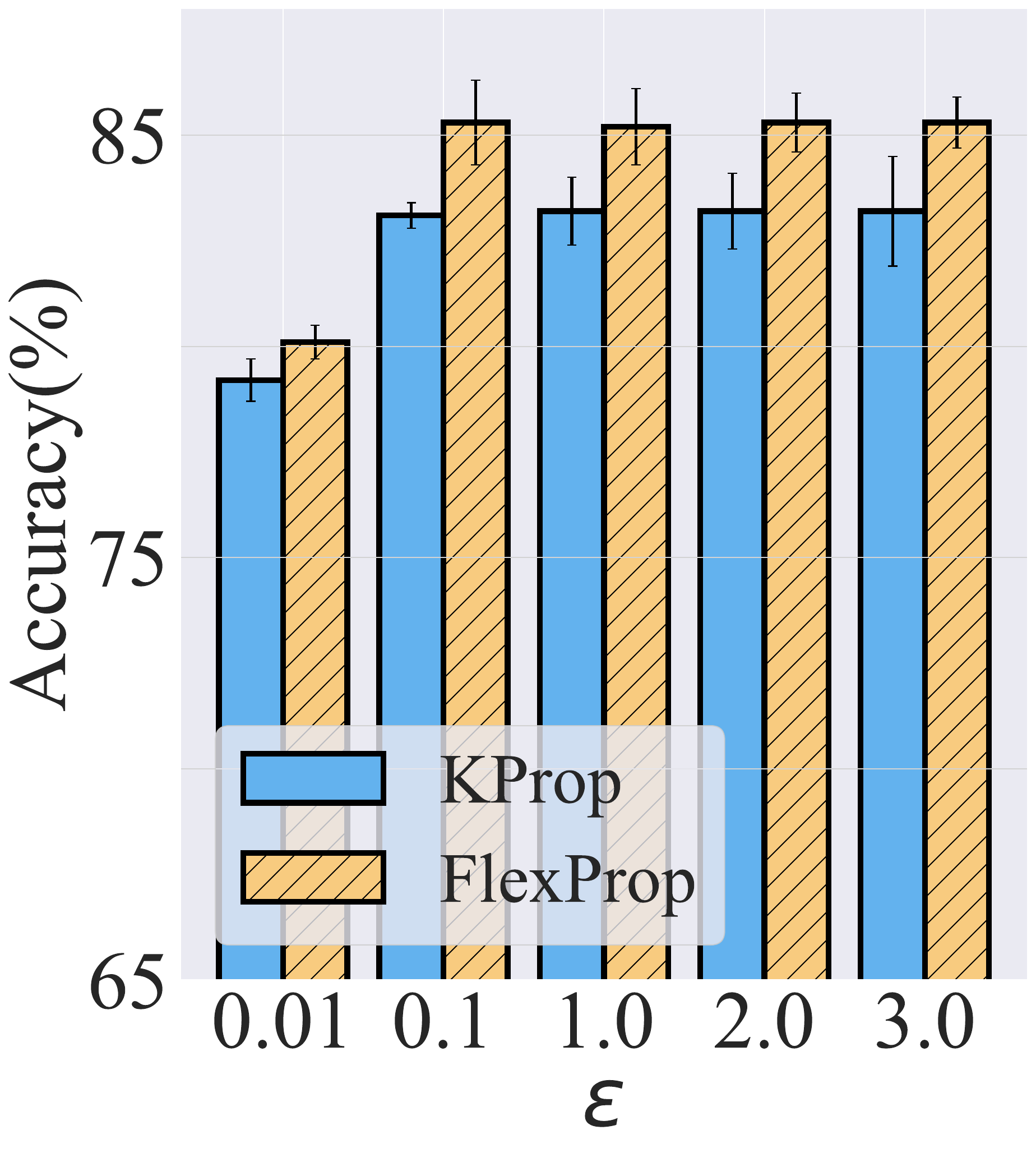} & 
	\includegraphics[width=0.14\linewidth]{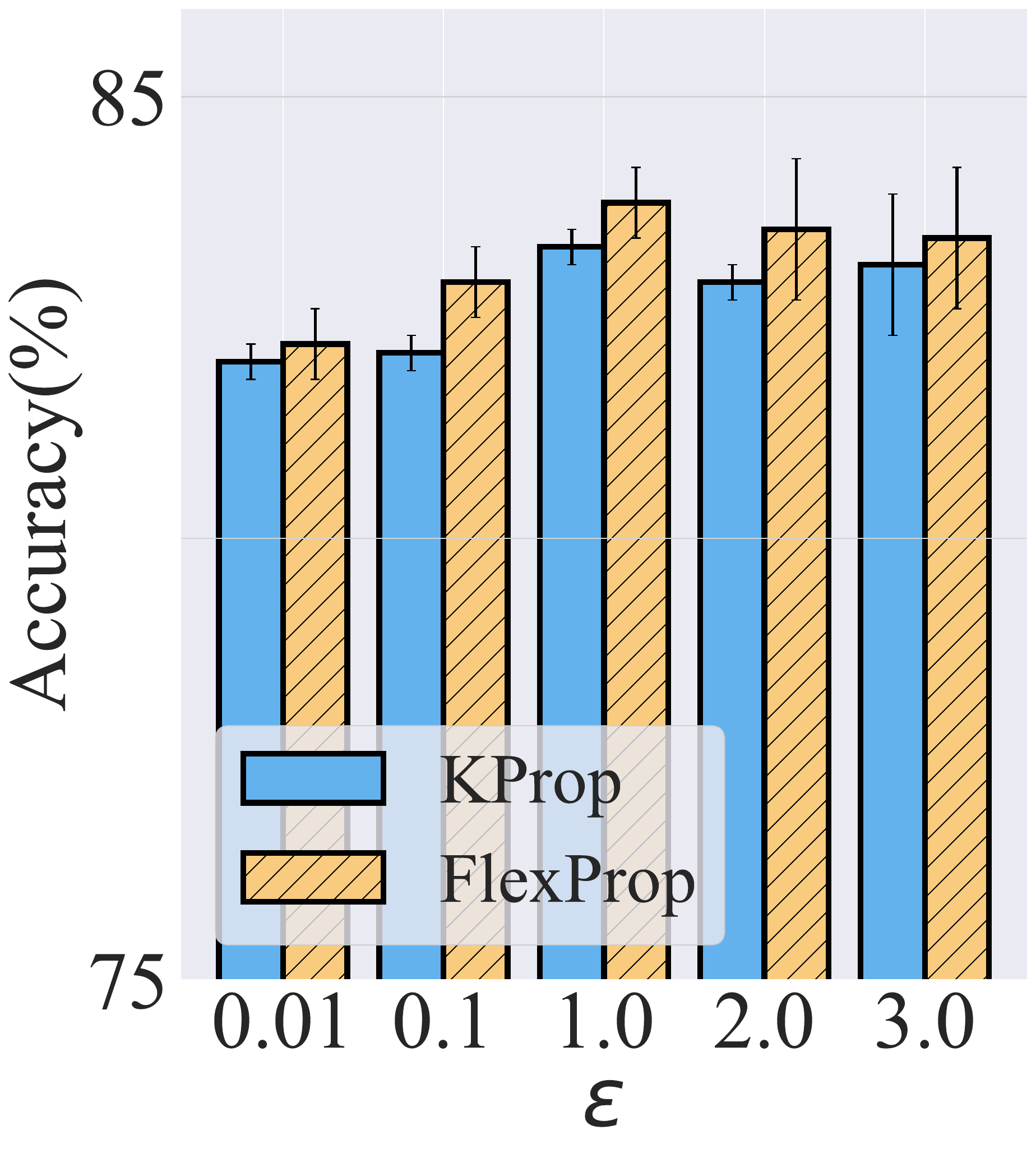} & 
	\includegraphics[width=0.14\linewidth]{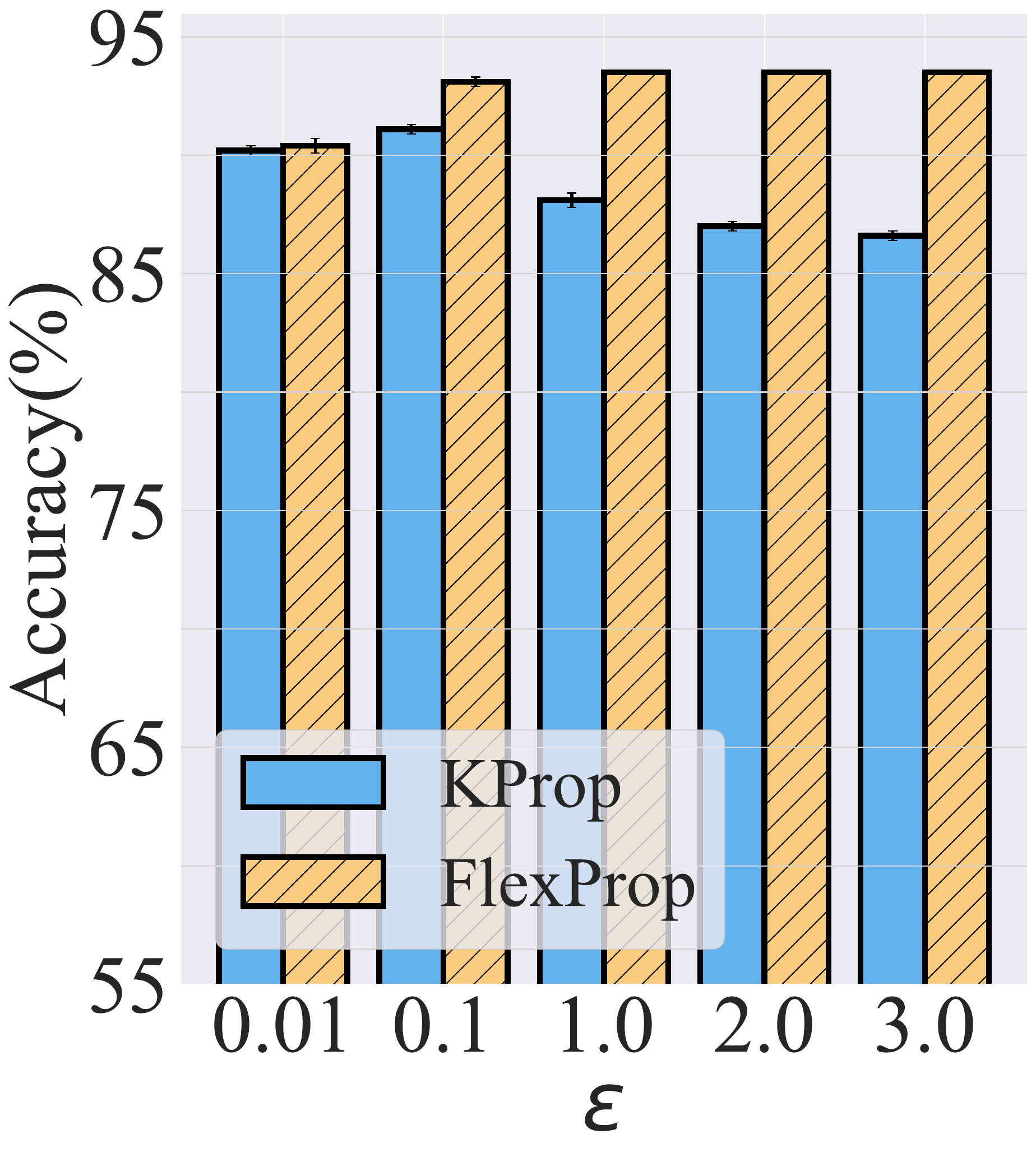} &
    \includegraphics[width=0.14\linewidth]{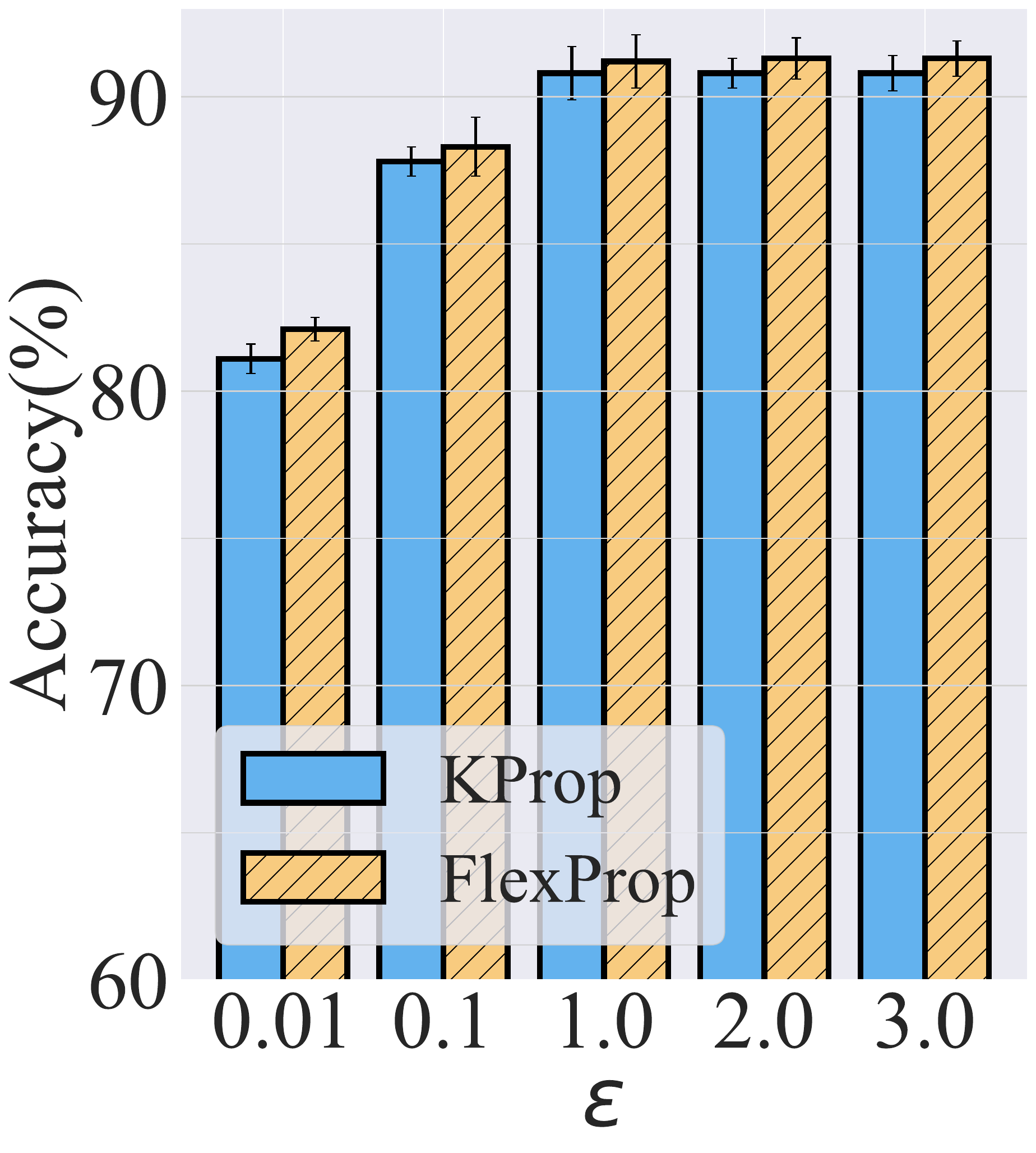}  \vspace{0.5em}\\
(a) Cora  & (b) CiteSeer  &(c) Pubmed  & (d) LastFM & (e) Facebook & (f) Wikipedia \\

	\end{tabular}
	\caption{Effect of FlexProp vs. KProp on graph learning performance across various datasets using the GraphSAGE model.}
	\label{fig:66}
    \vspace{-1em}
\end{figure*}

\begin{figure*}[h!]
	\centering
	\begin{tabular}{cccccc}
  \multicolumn{6}{c}{\textbf{}} \\
   \includegraphics[width=0.14\linewidth]{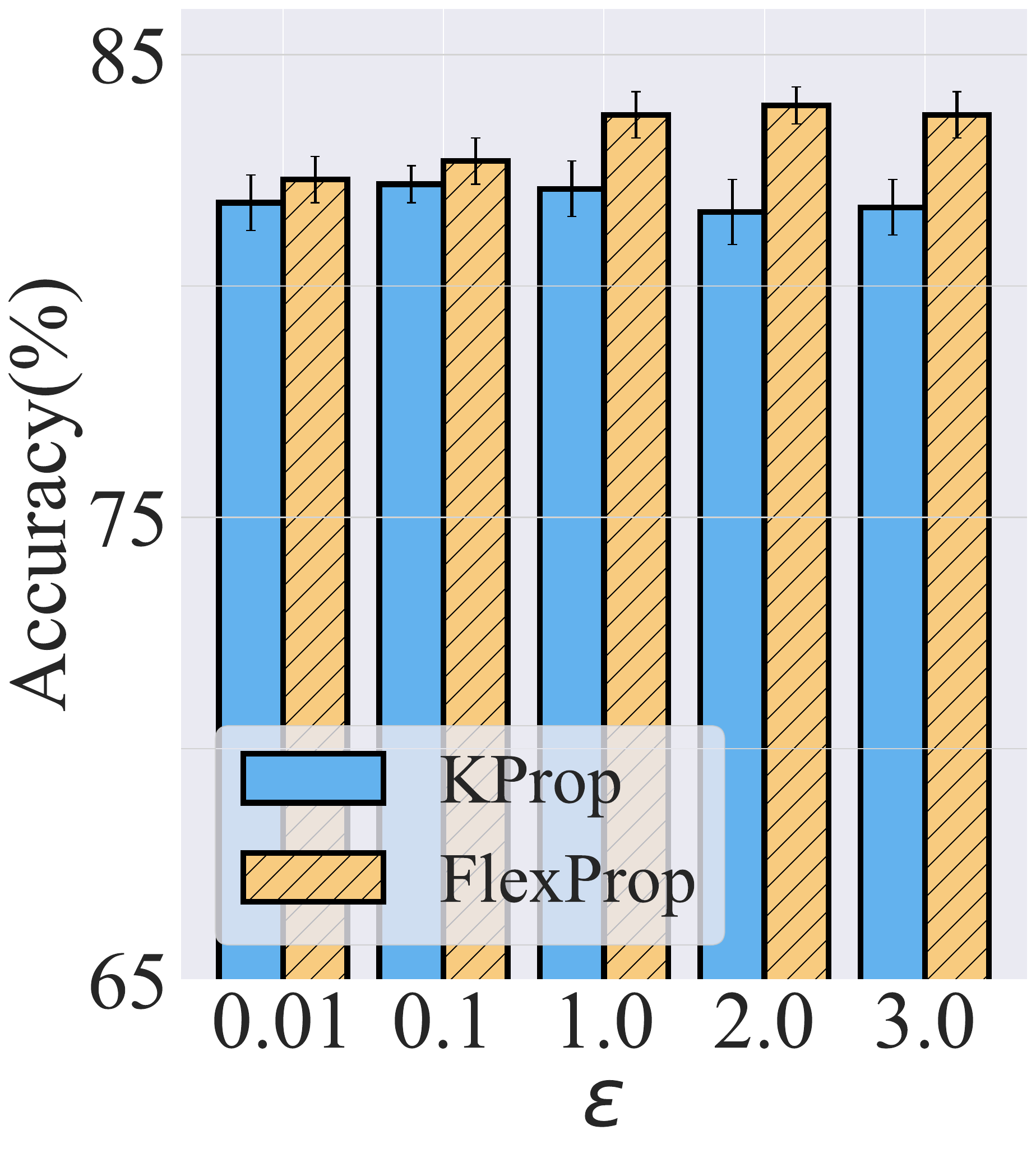} & 
   \includegraphics[width=0.14\linewidth]{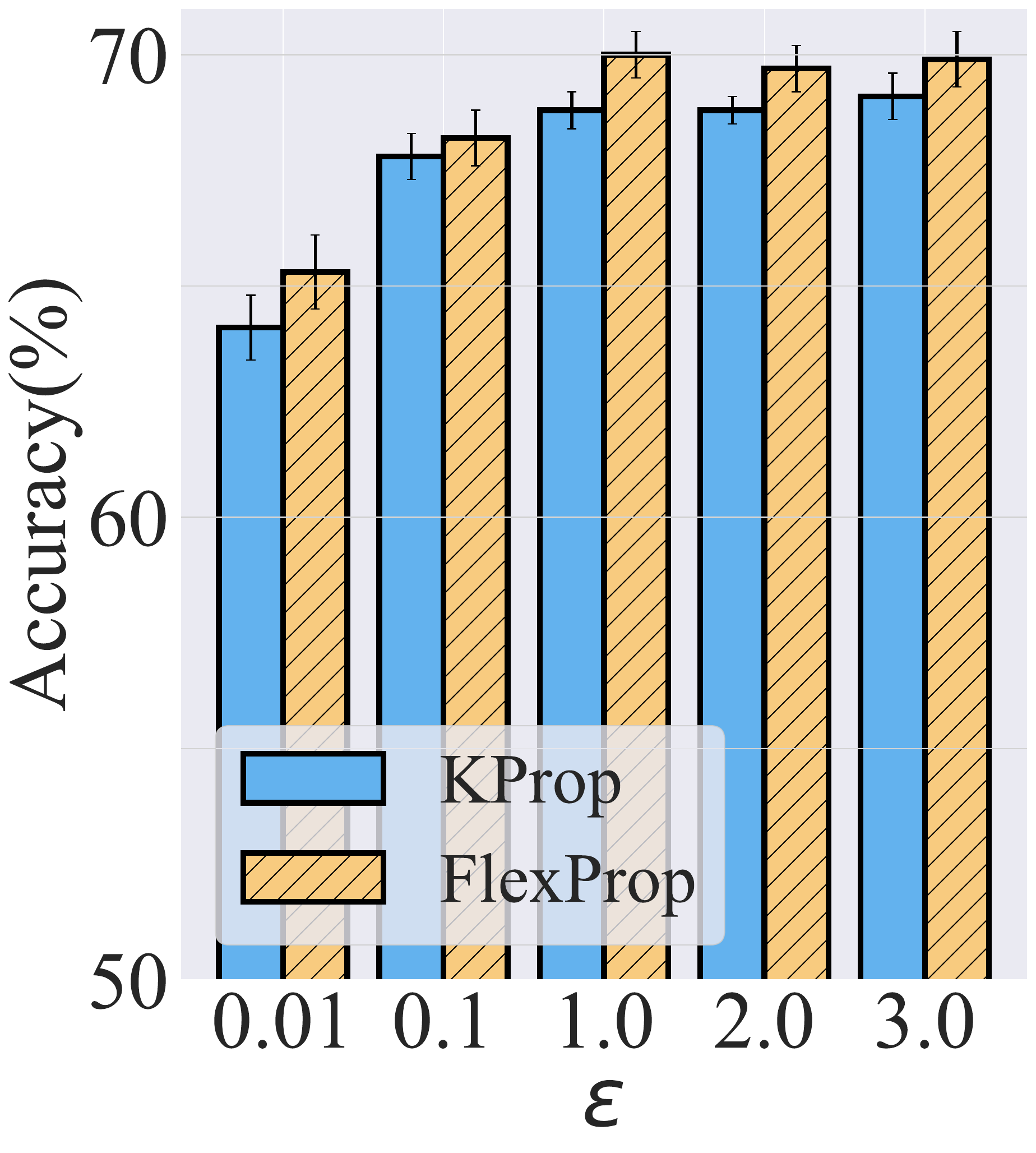} &
	\includegraphics[width=0.14\linewidth]{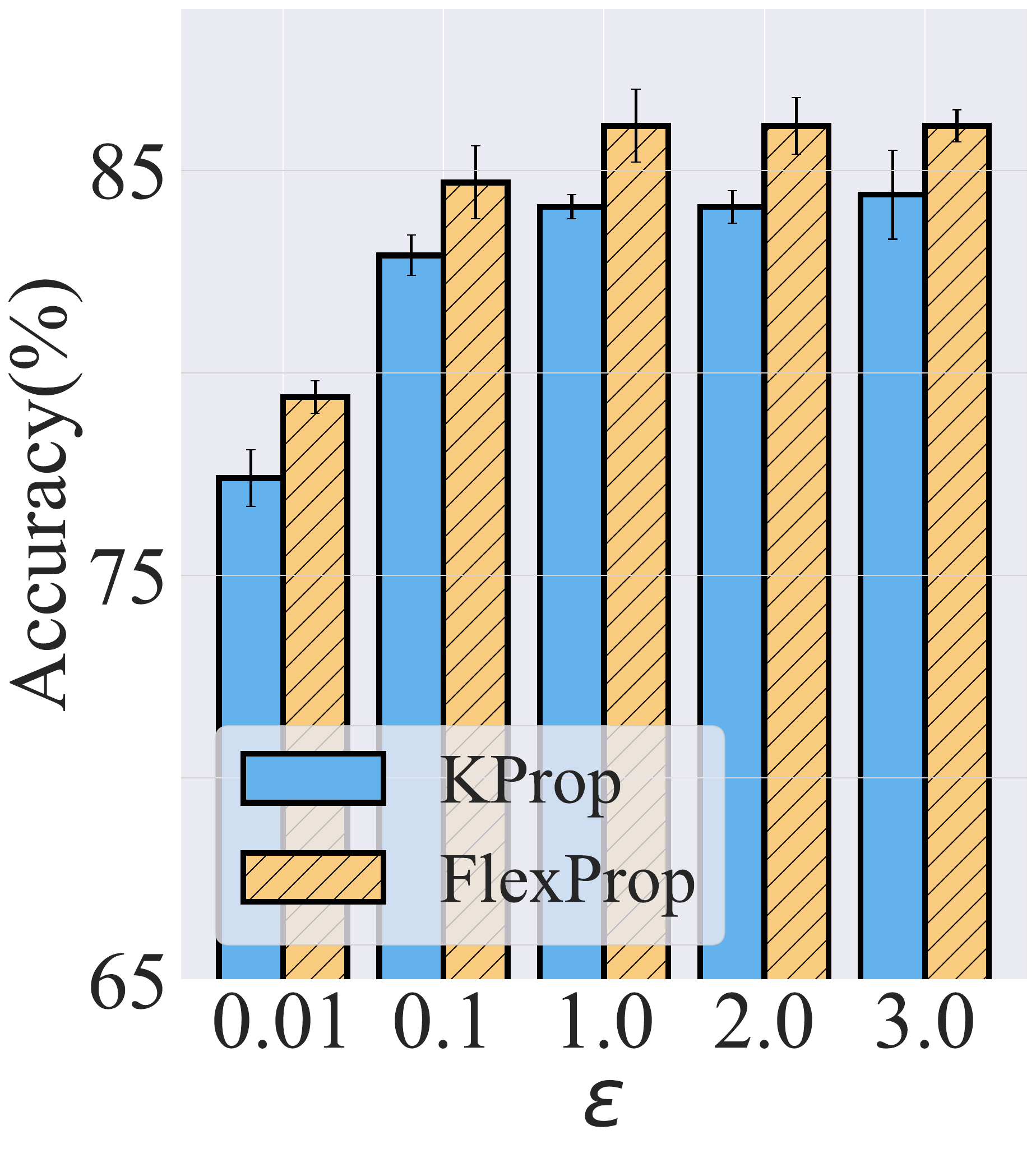} & 
	\includegraphics[width=0.14\linewidth]{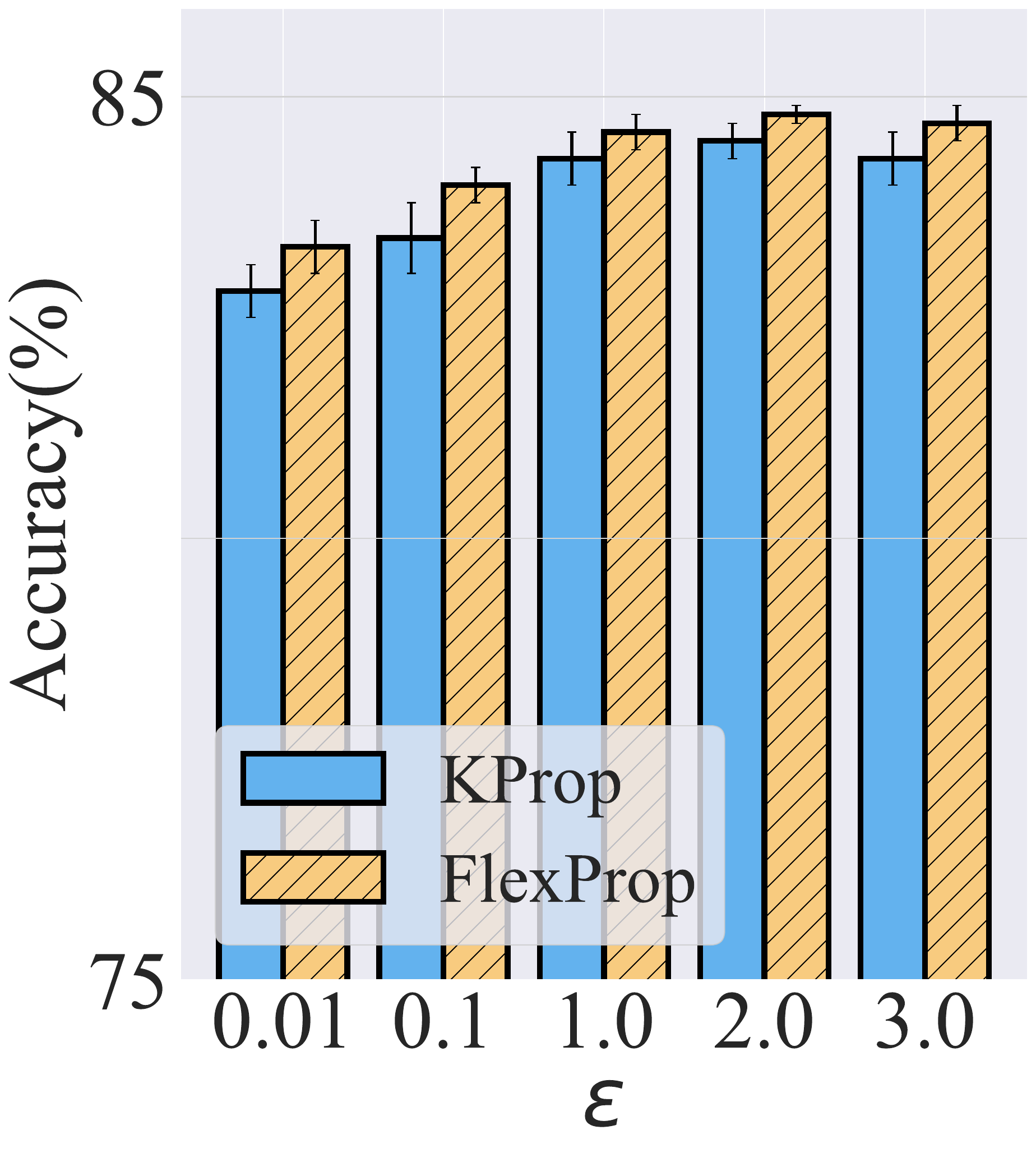} & 
	\includegraphics[width=0.14\linewidth]{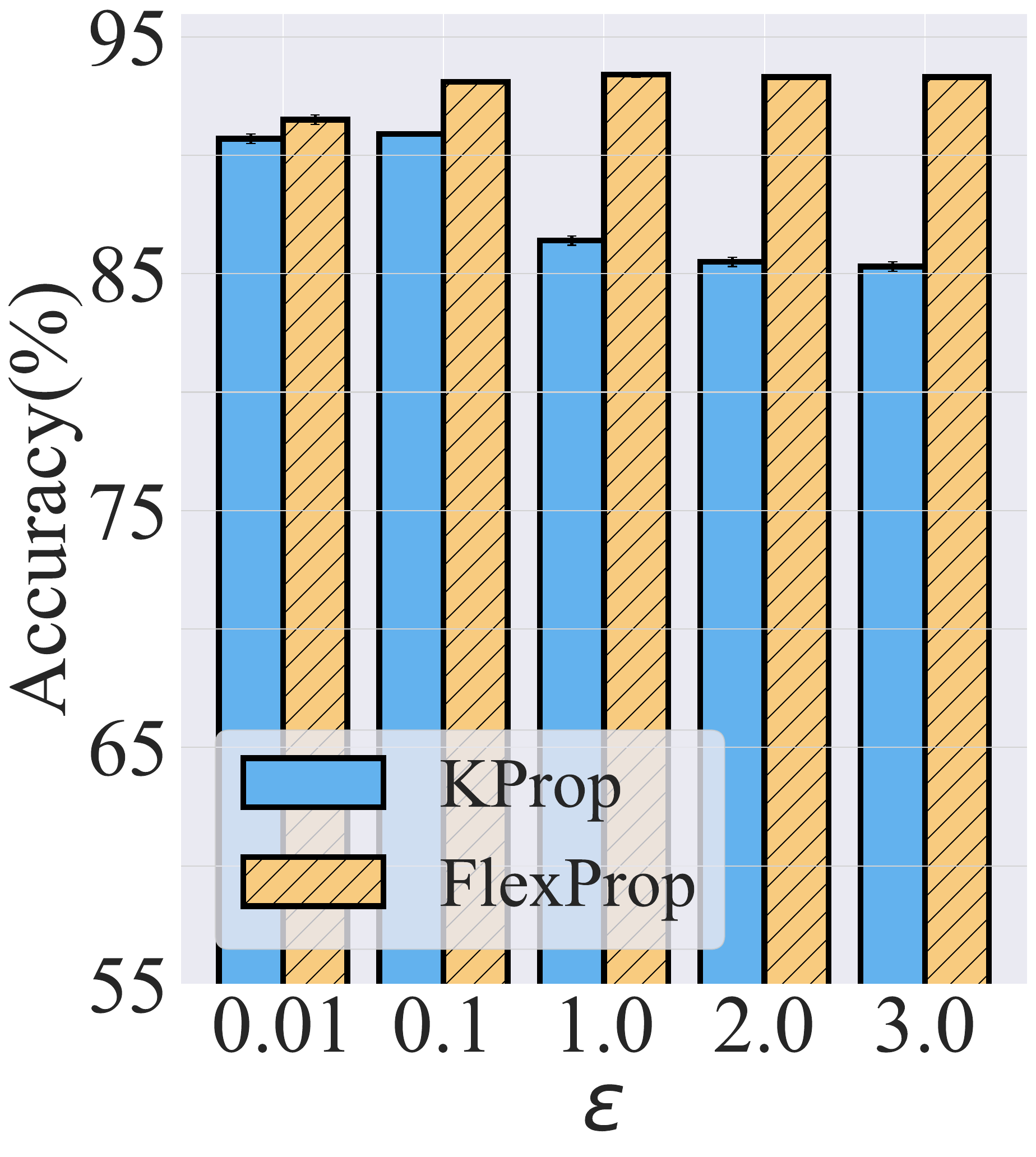} &
    \includegraphics[width=0.14\linewidth]{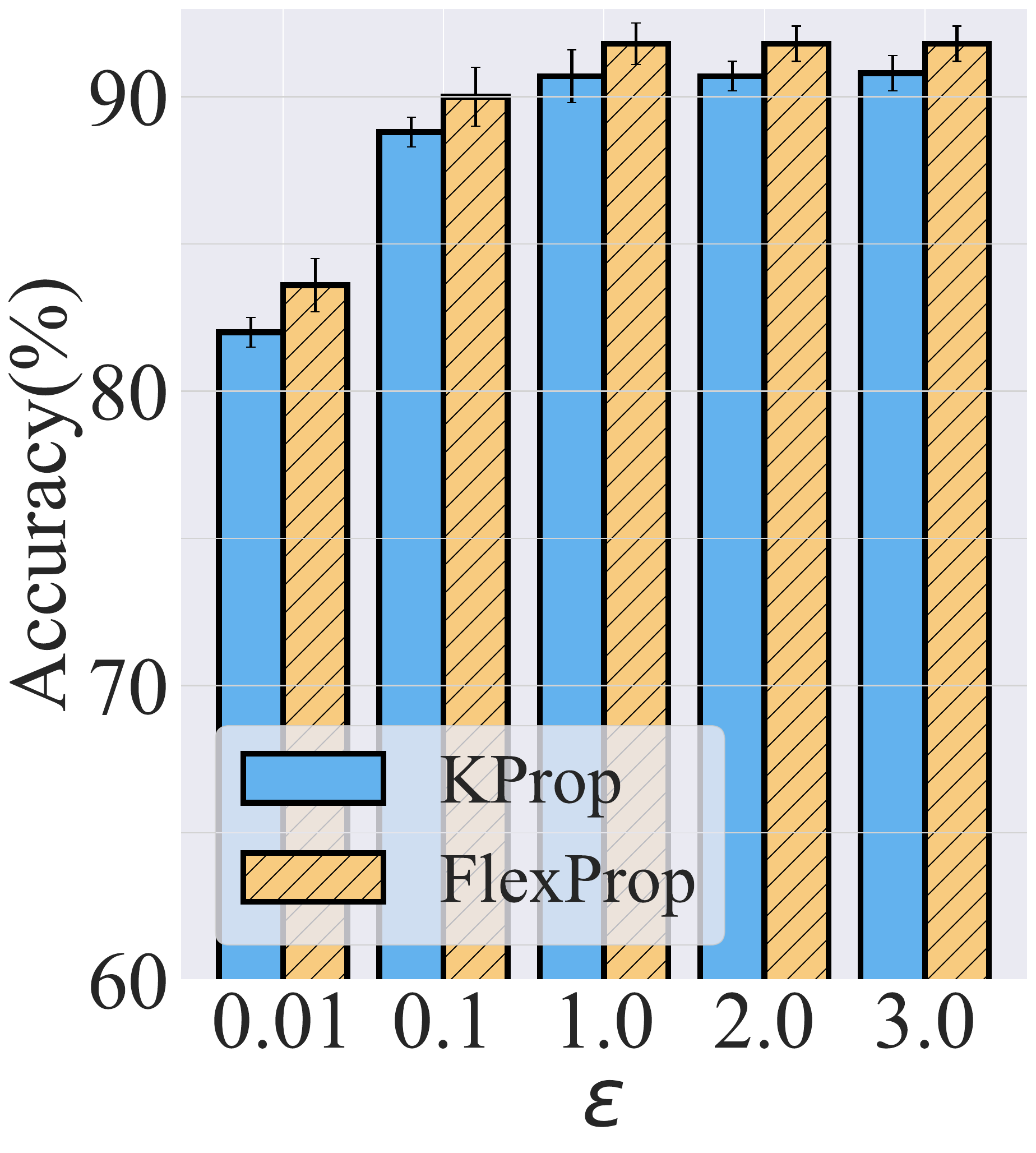} \\
  (a) Cora  & (b) CiteSeer  &(c) Pubmed  & (d) LastFM & (e) Facebook & (f) Wikipedia  \\
  
	\end{tabular}
	\caption{Effect of FlexProp vs. KProp on graph learning performance across various datasets  using the GCN model.}
	\label{fig:77}
\end{figure*}

\subsubsection{On the effects of $\gamma$}\label{gamma}
Fig.~\ref{fig:5}(b) illustrates the accuracy of PPGNN based on the GraphSAGE and the Cora dataset, considering privacy budget allocation parameter $\gamma\in\{0.1,0.3,0.5,0.7,0.9\}$ and privacy budget $\epsilon\in\{0.01,1.0\}$. The results reveal a non-monotonic trend: model accuracy initially increases with $\gamma$, reaches a peak near the middle values, then decreases as $\gamma$ continues to grow. This phenomenon reflects a fundamental trade-off in privacy budget allocation. When $\gamma$ is too small, the privacy budget assigned to protecting user privacy levels is insufficient, causing large noise in the privacy level perturbation. This undermines the effectiveness of the server-side FlexProp noise calibration, ultimately reducing model utility. Conversely, when $\gamma$ is too large, less budget remains for perturbing node features, leading to excessive distortion of node feature data and hence lower accuracy. Therefore, carefully balancing the allocation between protecting privacy levels and node features is crucial. Our experiments suggest that intermediate $\gamma$ values, typically around 0.5, tend to provide the best trade-off, achieving higher utility while maintaining adequate privacy protections.
\subsubsection{On the effects of $K$}\label{gamma}
Fig.~\ref{fig:kkk} displays the accuracy of PPGNN based on the GraphSAGE as the GNN backbone model, considering privacy budget $\epsilon\in\{0.01, 1.0\}$ and step parameter $K\in\{0,2,4,8,16\}$. From Fig.~\ref{fig:kkk}, as $K$ increases, the node classification accuracy initially improves. This improvement is attributed to FlexProp algorithm effectively calibrating noisy node features by considering information from multi-hop neighbors. However, as $K$ continues to increase, accuracy eventually decreases due to over-smoothing of output vectors caused by aggregating messages from too distant nodes, impacting FlexProp's denoising accuracy. Therefore, selecting an appropriate value of $K$ is crucial.

\subsubsection{On the effects of $m$}
Fig.~\ref{fig:5}(c) shows the accuracy of PPGNN on Cora and LastFM with varying $m$. As $m$ increases, the effective privacy budget grows, resulting in reduced noise per dimension and thus improved utility. Our experiments on the Cora dataset confirm that the proposed method consistently outperforms baselines across different values of $m$.

\section{Related Works}\label{S6}

\subsection{Centralized DP-based Graph Learning}
In the centralized DP setting, a trusted curator has full access to the raw user data and adds noise during the training process to ensure privacy. Several works have explored graph learning under this framework. Wu \textit{et al}.~\cite{wu2022linkteller} study the edge re-identification attacks on GNNs and propose \textsc{DpGCN}, a centralized DP mechanism ensuring edge-level privacy. Kolluri \textit{et al}.~\cite{kolluri2022lpgnet} propose a new neural network architecture called \textsc{LPGNet} for training graphs with privacy-sensitive edges. Daigavane \textit{et al}.~\cite{xiang2024preserving} extend the well-known DP-SGD~\cite{abadi2016deep} algorithm to GNN to achieve stronger node-level central DP. Sajadmanesh \textit{et al}.~\cite{sajadmanesh2023gap} propose a novel differentially private GNN, called GAP, which safeguards the privacy of nodes and edges using aggregation perturbation. Sajadmanesh \textit{et al}.~\cite{sajadmanesh2023progap} also propose a novel differentially private GNN called ProGAP, which improves the trade-off between accuracy and privacy by using a progressive training scheme. These approaches represent significant progress in centralized graph privacy, but they rely on a trusted data collector—an assumption that may not hold in many decentralized applications.
\vspace{-1em}
\subsection{Local DP-based Graph Learning}
In contrast, LDP~\cite{xiong2020comprehensive,yang2024local,he2025mitigating} assumes no trusted aggregator, and each user perturbs their own data before sharing. This provides stronger privacy but poses greater challenges for utility preservation. Sajadmanesh \textit{et al}.~\cite{sajadmanesh2021locally} introduce LPGNN, the first LDP-based method for graph learning with private node features and labels. Lin \textit{et al.}~\cite{lin2022towards} propose Solitude, a decentralized LDP framework that protects both graph structure and features. Further, Pei \textit{et al}.~\cite{pei2023privacy} consider privacy-enhanced issues in the context of local subgraphs. To achieve a trade-off between privacy-preserving graph topological relations and graph learning utility in locally private graph learning, Hidano \textit{et al}.~\cite{hidano2022degree} proposed a method for training GNN models with a novel LDP algorithm called DPRR, which provides degree protection for edges of graph nodes. For edge preservation, Zhu \textit{et al}.~\cite{zhu2023blink} introduced the \textsc{Blink}. \textsc{Blink} involves injecting noise into the adjacency lists and node degree to achieve privacy preservation for edges, respectively. However, none of these works consider the personalized privacy requirements of users, which is crucial in practical applications where users may have different privacy concerns. In contrast, our work addresses this gap by proposing PPGNN, the first locally differentially private GNN framework that enables personalized privacy protection across users while retaining utility. Our approach allows each user to independently choose a privacy level and leverages a flexible aggregation mechanism to mitigate the resulting noise.
\vspace{-1em}
\subsection{Personalized Local Differential Privacy}
To accommodate users with personalized privacy preferences, a number of works have extended LDP mechanisms into the personalized setting. Early studies such as~\cite{chen2016private,yiwen2018utility,murakami2019utility} designed utility-optimized mechanisms for spatial data and histogram estimation, where each user may specify a different privacy budget. Shen \textit{et al.}~\cite{shen2021pldp} proposed PLDP for aggregating multi-dimensional data, while Zeng \textit{et al.}~\cite{zeng2022collecting} studied mean estimation under personalized privacy. More recently, LDPTube~\cite{duan2024ldptube} provides theoretical benchmarks for personalized LDP in high-dimensional spaces, and He \textit{et al.}~\cite{he2025personalized} focused on range queries over mobile user data under heterogeneous privacy settings. While these methods demonstrate the importance of personalization in LDP, they focus exclusively on \textit{data collection and aggregation tasks}. In contrast, our work addresses privacy-preserving representation learning in graphs, where local perturbations propagate through network structures and interact with diverse noise scales. Our proposed FlexProp further adapts message passing to account for personalized noise during multi-hop aggregation, which has not been considered in prior PLDP work. To the best of our knowledge, this is the first personalized LDP framework designed for learning representations on graphs.

\vspace{-1em}
\subsection{Privacy-Preserving Graph Learning} 
Beyond DP, several paradigms have been explored for privacy-preserving graph learning. Federated learning~\cite{wang2023automated,zhang2021survey} enables collaborative GNN training across decentralized clients without sharing raw data, but lacks formal privacy guarantees unless combined with DP. Adversarial training-based methods~\cite{hsieh2021netfense} aim to suppress sensitive attribute leakage, while graph perturbation approaches~\cite{liao2021information} obfuscate graph structure to reduce privacy risks. In contrast, our work adopts LDP, which provides provable privacy guarantees without relying on a trusted aggregator, and further extends it to the personalized setting to accommodate heterogeneous user privacy preferences.

\section{Conclusion}\label{S7}
In this paper, we introduce PPGNN (\underline{P}ersonalized \underline{P}rivacy-preserving \underline{G}raph \underline{N}eural \underline{N}etwork), a locally differentially private graph learning framework tailored to meet users' personalized privacy requirements. To safeguard user features while aligning with these personalized privacy preferences, we propose a novel LDP mechanism, the Personalized Perturbation Mechanism. It comprises the Multi-dimensional Local Randomizer, which protects multi-dimensional node features with utility guarantees, and the Extended Square Wave mechanism, which safeguards users' privacy levels across discrete domains effectively. Furthermore, to enhance the efficiency of noise calibration in the perturbed features, we present FlexProp, a weighted aggregation algorithm. Extensive experiments on six real-world datasets demonstrate the effectiveness and robustness of both PPGNN and its key components. 

\textit{Limitations:} While PPGNN demonstrates strong utility across diverse settings, we acknowledge two limitations. First, under highly heterogeneous privacy budget distributions, nodes with very low privacy budgets inject substantially more noise, and when such nodes dominate a target node's neighborhood $\mathcal{N}(v)$, the FlexProp aggregation may suffer from residual bias despite the privacy-aware weighting scheme $\mathbf{\Pi}$. Second, the framework assumes users can explicitly declare a privacy level $\tau \in \{1, \dots, h\}$, which may not always hold in practice. Developing adaptive budget calibration strategies or inferring privacy preferences from behavioral signals remains a promising direction for future work.

\textit{Applicable Scenarios:} PPGNN is particularly well-suited for decentralized graph learning settings where users have heterogeneous privacy preferences, such as social platforms, mobile recommendation systems, and federated user profiling. In these contexts, a one-size-fits-all privacy budget would either over-protect low-sensitivity users (sacrificing utility) or under-protect high-sensitivity users (compromising privacy). PPGNN addresses this by enabling user-specific privacy budgets while preserving analytical utility through adaptive noise calibration.

\bibliographystyle{IEEEtran}
\bibliography{IEEEabrv,bibtex/bib/IEEEexample}

\ifCLASSOPTIONcaptionsoff
  \newpage
\fi

\end{document}